\documentclass{article}
\usepackage{arxiv}

\usepackage{microtype}
\usepackage{graphicx}
\usepackage{subfigure}
\usepackage{booktabs} %

\usepackage{amsfonts}
\usepackage{amsmath}
\usepackage{amssymb}
\usepackage{amsthm}
\usepackage{mathtools}
\usepackage{algorithmic}
\usepackage{algorithm}
\usepackage{nicefrac}
\usepackage{xcolor}
\usepackage{caption}
\usepackage{bm}
\usepackage{tikz}
\usepgflibrary{shapes.geometric}
\usetikzlibrary{calc, shapes, positioning, arrows}
\usepackage{hyperref}
\usepackage{cleveref}

\graphicspath{ {./Figures/} }

\newcommand{\DO}{\mathcal{D}}

\newcommand{\Models}{\mathcal{S}}
\newcommand{\Ex}{\mathbb{E}}

\newcommand{\R}{\mathbb{R}}

\newcommand{\RKM}{{\normalfont \textsc{rkm}}}

\DeclareMathOperator*{\argmin}{arg\,min}

\newtheorem{theorem}{Theorem}

\newtheorem{corollary}[theorem]{Corollary}

\newtheorem{definition}[theorem]{Definition}

\newtheorem{lemma}[theorem]{Lemma}

\newtheorem{proposition}[theorem]{Proposition}

\title{Robust Unsupervised Learning via L-Statistic Minimization}

\author{
Andreas Maurer \\
Istituto Italiano di Tecnologia\\
\texttt{am@andreas-maurer.eu} \\
\And
Daniela A. Parletta \\
Istituto Italiano di Tecnologia \&\\
University of Genoa\\
\texttt{daniela.parletta@iit.it} \\
\AND
Andrea Paudice \\
Istituto Italiano di Tecnologia \& \\
University of Milan\\
\texttt{andrea.paudice@iit.it} \\
\And
Massimiliano Pontil \\
Istituto Italiano di Tecnologia \& \\
University College of London\\
\texttt{massimiliano.pontil@iit.it} \\
}

\begin{document}
\maketitle

\begin{abstract}
Designing learning algorithms that are resistant to perturbations of the underlying data distribution is a problem of wide practical and theoretical importance. We present a general approach to this problem focusing on unsupervised learning. The key assumption is that the perturbing distribution is characterized by larger losses relative to a given class of admissible models. This is exploited by a general descent algorithm which minimizes an $L$-statistic criterion over the model class, weighting small losses more. Our analysis characterizes the robustness of the method in terms of bounds on the reconstruction error relative to the underlying unperturbed distribution. As a byproduct, we prove uniform convergence bounds with respect to the proposed criterion for several popular models in unsupervised learning, a result which may be of independent interest.Numerical experiments  with \textsc{kmeans} clustering and principal subspace analysis demonstrate the effectiveness of our approach.
\end{abstract}

\section{Introduction}
\label{intro}
{Making learning methods robust is a fundamental problem in machine learning and statistics.} In this work we proposes an approach to unsupervised learning which is resistant to unstructured contaminations of the underlying data distribution. As noted by Hampel \cite{Hampel2001}, ``outliers" are an ill-defined concept, and an approach to robust learning, which relies on rules for the rejection of outliers (see \cite{Keith1996} and the references therein) prior to processing may be
problematic, since the hypothesis class of the learning process itself may determine which data is to be regarded as structured or unstructured. Instead of the elimination of outliers -- quoting Hampel ``data that don't fit the pattern set by the majority of the data" -- in this paper we suggest to restrict attention to ``a \textit{sufficient portion} %
\textit{of the data in good agreement with one of the hypothesized models}''.

{To implement the above idea, we propose using $L$-estimators \cite{Serfling1980}, which are formed by a weighted average of the order statistics. That is, given a candidate model, we first rank its losses on the empirical data and than take a weighted average which emphasizes small losses more. An important example of this construction is the average of a fraction of the smallest losses. However, our observations apply to general classes of weight functions, which are only restricted to be non-increasing and in some cases Lipschitz continuous.}

We highlight that although $L$-statistics have a long tradition, a key novelty of this paper is to use them as objective functions based on which to search for a robust model. This approach is general in nature and can be applied  
to robustify any learning method, supervised or unsupervised, based on empirical risk minimization. In this paper we focus 
on unsupervised learning, and our analysis includes \textsc{kmeans} clustering, principal subspace analysis and sparse coding, among others. 

This paper makes the following contributions:
\vspace{-.1truecm}
\begin{itemize}
\vspace{-.1truecm}
\item A theoretical analysis of the robustness of the proposed method (Theorem 1). Under the assumption that the data-distribution is a mixture of an unperturbed distribution adapted to our model class and a perturbing distribution, we identify conditions under which we can bound the reconstruction error, when the minimizer of the proposed objective trained from the perturbed distribution is tested on the unperturbed distribution. 
\vspace{-.1truecm}
\item An analysis of generalization (Theorems 4--6). We give dimension-free uniform bounds in terms of Rademacher averages as well as a dimension- and variance-dependent uniform bounds in terms of covering numbers which can outperform the dimension-free bounds under favorable conditions.
\vspace{-.1truecm}
\item A meta-algorithm operating on the empirical objective which can be used whenever there is a descent algorithm for the underlying loss function (Theorem 9).
\end{itemize}
\vspace{-.2truecm}

The paper is organized as follows. In Section \ref{sec:2} we give a brief overview of unsupervised (representation) learning. In Sections \ref{sec:3} to \ref{sec:5} we present and analyze our method. In Section \ref{sec:4} we discuss an algorithm optimizing the proposed objective and in Section \ref{sec:exps} we present numerical experiments
with this algorithm for \textsc{kmeans} clustering and principal subspace analysis, which indicate that the proposed method is promising. Proofs can be found in the supplementary material.

\paragraph{Previous Work} Some elements of our approach have a long tradition. For fixed models the proposed empirical objectives are called $L$-statistics or $L$-estimators. They have been used in robust statistics since the middle of the last century  \cite{LLoyd1952} and their asymptotic properties have been studied by many authors (see \cite{Serfling1980} and the references therein). Although influence functions play a certain role, our approach is somewhat different from the traditions of robust statistics. Similar techniques to ours have been experimentally explored in the context of classification \cite{Han2018} or latent variable selection \cite{Kumar2010}. Finite sample bounds, uniform bounds, the minimization of $L$-statistics over model classes and the so called risk based-objectives however are more recent developments~
\cite{Maurer2018,Maurer2019u,Lee2020}, and we are not aware of any other general bounds on the reconstruction error of models trained from perturbed data. %
A very different line of work for robust statistics are model-independent methods available in high dimensions \cite{Elmore2006,Fraiman2019}. {Although elegant and very general, these depth-related pre-processing methods may perform sub-optimally in practice, as our numerical experiments indicate.}
Finally, we note that previous work on PAC learning (e.g. \cite{AngluinL87}) has addressed the problem of learning a good classifier with respect to a target, when the data comes from a perturbed distribution affected by unstructured noise. Similarly to us, they consider that the  target distribution is well adapted to the model class.

\section{Unsupervised Learning}
\label{sec:2}
Let $\mathcal{S}$ be a class of subsets of $\mathbb{R}^{d}$, which we call the model class. %
For $S\in \mathcal{S}$ define the distortion function $d_{S}:\mathbb{R}^{d}\rightarrow \left[ 0,\infty \right)$ by\footnote{In most parts our analysis applies also to other distortion measures, for example omitting the square in (\ref{Distortion function}). The chosen form is important for generalization bounds, when we want to bound the complexity of the class $\left\{ x\mapsto d_{S}\left( x\right) :S\in \mathcal{S}\right\}$ for specific cases.}
\begin{equation}
d_{S}\left( x\right) =\min_{y\in S}\left\Vert x-y\right\Vert ^{2}\text{ for }%
x\in \mathbb{R}^{d}.  
\label{Distortion function}
\end{equation}
We assume that the members of $\mathcal{S}$ are either compact sets or subspaces, so the minimum in (\ref{Distortion function}) is always attained. For instance $\mathcal{S}$ could be the class of singletons, a class of subsets of cardinality $k$, the class of subspaces of dimension $k$, or a class of compact convex polytopes with $k$ vertices\footnote{In these cases the set $S$ is the image of a linear operator on a prescribed set of code vectors, see \cite{Maurer2010}. Our setting is more general, e.g. it includes non-linear manifolds.
}. 

We write $\mathcal{P}\left( \mathcal{X}\right) $ for the set of Borel probability measures on a locally compact Hausdorff space $\mathcal{X}$. If $\mu \in \mathcal{P}\left( \mathbb{R}^{d}\right) $, define the probability measure $\mu _{S}\in \mathcal{P}\left( \left[ 0,\infty \right) \right) $ as the push-forward of $\mu $ under $d_{S}$, that is, $\mu _{S}\left( A\right) =\mu\left(\{ x:d_{S}\left( x\right) \in A\}\right)$ for $A\subseteq \left[0,\infty \right) $. Now consider the functional $\Phi :\mathcal{P}\left(\left[ 0,\infty \right) \right) \mathcal{\rightarrow }\left[ 0,\infty\right) $ defined by

\begin{equation}
\Phi \left( \rho \right) =\int_{\text{0}}^{\infty }rd\rho \left( r\right),~~~\rho \in \mathcal{P}\text{.}
\label{eq:functional}
\end{equation}
Then $\Phi \left( \mu _{S}\right) =\mathbb{E}_{X\sim \mu }\left[ d_{S}\left(X\right) \right] $ is the expected \textit{reconstruction error}, incurred when coding points by the nearest neighbors in $S$. The measures $\mu_{S}\in \mathcal{P}\left( \left[ 0,\infty \right) \right) $ and the functional $\Phi $ allow the compact and general description of several problems of unsupervised learning as 
\begin{equation}
\min_{S\in \mathcal{S}}\Phi \left( \mu _{S}\right) 
=\min_{S\in \mathcal{S}}\mathbb{E}_{X\sim \mu }\left[ d_{S}\left( X\right) \right].   
\label{Unsupervised objective}
\end{equation}
Denote with $S^{\ast }=S^{\ast }\left( \mu \right) $ a global minimizer of (\ref{Unsupervised objective}). Returning to the above examples, if $\mathcal{S}$ is the class of singleton sets, then $S^{\ast }\left( \mu \right) $ is the mean of $\mu $. If it is the class of subsets of cardinality $k$, then $S^{\ast }\left( \mu \right)$ is the optimal set of centers for \textsc{kmeans} clustering. If $\mathcal{S}$ is the class of $k$-dimensional subspaces, then $S^{\ast }\left( \mu \right)$ is the principal $k$-dimensional subspace. 

An important drawback of the above formulation is that the functional $\Phi $ is very sensitive to perturbing masses at large distortions $R$. In the tradition of robust statistics (see e.g. \cite{Hampel1974,Serfling1980}) this can be expressed in terms of the 
\textit{influence function}, {measuring the effect of an infinitesimal point mass perturbation of the data}. Let $\delta _{R}$ be the unit mass at $R>0$, then the influence function 
\begin{eqnarray}
{\rm IF}\left( R;\rho ,\Phi \right)  &:=&\frac{d}{dt}\Phi \big((1-t)
\rho +t\delta _{R}\big) \bigg|_{t=0}   \label{Influence non-robust} = R-\Phi \left( \rho \right) ,  \nonumber
\end{eqnarray}
can be arbitrarily large, indicating that even a single datapoint could already corrupt $S^{\ast }\left( \mu \right)$. To overcome this problem, in the next section we introduce a class of robust functional based on $L$-statistics.

\section{Proposed Method}
\label{sec:3}
Our goal is to minimize the reconstruction error on unperturbed test data, from perturbed training data. Specifically, we assume that the data we observe comes from a perturbed distribution $\mu$ that is the mixture of an unperturbed distribution $\mu^{\ast }$, which is locally concentrated on the minimizer $S^{\ast }=S^{\ast}\left( \mu ^{\ast }\right)$, and a perturbing distribution $\nu $ which is unstructured in the sense that it does not concentrate on any of our models\footnote{This is in contrast with the assumptions made in adversarial learning, where the goal is to increase robustness against adversarial worst-case perturbations (see e.g. \cite{Ragi2018}).}. Figure \ref{fig:ex_mean} depicts such a situation, when $\mathcal{S}$ is the set of singletons and $d=1$.

\begin{figure}[ht]
\begin{center}
\begin{tikzpicture}[scale=1.6, domain=-1:3]
\draw [green, thick, domain=-1:3, samples=500] plot (\x, {1/(0.3 * sqrt(2 * pi)) * exp(-(\x/0.3)*(\x/0.3))}); %
\draw [red, thick, domain=-1:3, samples=500] plot (\x, {1/(1 * sqrt(2 * pi)) * exp(-((\x-2)/1)*((\x-2)/1))}); %
\draw [black, thick, domain=-1:3, samples=500] plot (\x, {0.4/(0.3 * sqrt(2 * pi)) * exp(-(\x/0.3)*(\x/0.3)) + 0.6/(1 * sqrt(2 * pi)) * exp(-((\x-2)/1)*((\x-2)/1))}); %
\filldraw[fill=none] (0,0) node[anchor=north] {$S^{\ast}$}; %
\end{tikzpicture}
\end{center}
\caption{Densities of the unperturbed distribution $\protect\mu ^{\ast }$ (light green) with high local concentration on the optimal model $S^{\ast }$%
, the perturbing distribution $\protect\nu $ (light red) without significant concentration, and the observable mixture $\protect\mu =\left( 1-\protect\lambda \right) \protect\mu ^{\ast }+\protect\lambda \protect\nu^{\ast }$ (black) at $\protect\lambda =0.6.$}
\label{fig:ex_mean}
\end{figure}
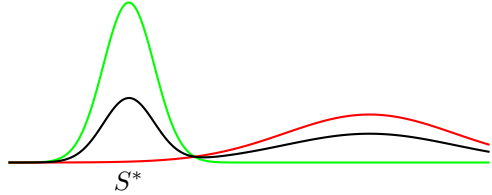

We wish to train from the available, perturbed data a model $\hat{S}\in\mathcal{S}$, which nearly minimizes the reconstruction error on the unperturbed distribution $\mu ^{\ast }$. To this end we exploit the assumption that the unperturbed distribution $\mu ^{\ast }$ is much more strongly concentrated at $S^{\ast }$ than the mixture $\mu =\left( 1-\lambda\right) \mu ^{\ast }+\lambda \nu $ is at models $S$ \emph{away} from $S^{\ast }$ in terms of reconstruction error.

The key observation is that if the mixture parameter $\lambda $ is not too large, the concentration of $\mu ^{\ast }$ causes the cumulative distribution function of the losses for the optimal model $F_{\mu _{S^{\ast}}}:r\mapsto \mu _{S^{\ast }}\left[ 0,r\right] $ to increase rapidly for small values of $r$, until it reaches the value $\zeta =F_{\mu _{S^{\ast}}}\left( r^{\ast }\right) $, where $r^{\ast }$ is a critical distortion radius depending on $S^{\ast }$. Thus, when searching for a model, we can consider as irrelevant the remaining mass $1-\zeta =\mu _{S^{\ast }}\left( r^{\ast },\infty \right) $, which can be attributed to $\nu $ and may arise from outliers or other contaminating effects. To achieve this, we modify the functional \eqref{eq:functional} so as to consider only the relevant portion of data, replacing $\Phi \left( \mu _{S}\right)$ by
\begin{equation}
\zeta ^{-1}\int_{0}^{F_{\mu _{S}}^{-1}\left( \zeta \right) }rd\mu _{S}\left(r\right) .
\label{eq:hard}    
\end{equation}
Intuitively, the minimization of \eqref{eq:hard} forces the search towards models with the smallest \emph{truncated} expected loss. Among such models there is also $S^*$, whose losses have the strongest concentration around a \emph{small} value and then leading to a very small value $r^\ast$ for $F_{\mu _{S^\ast}}^{-1}\left( \zeta \right)$.

More generally, since the choice of the hard quantile-thresholding at $\zeta $ is in many ways an ad hoc decision, we might want a more gentle transition of the boundary between relevant and irrelevant data. Let $W:\left[ 0,1\right]\rightarrow \left[ 0,\infty \right) $ be a bounded weight function and define, for every $\rho \in \mathcal{P}\left[ 0,\infty \right)$, 
\[
\Phi _{W}\left( \rho \right) =\int_{\text{0}}^{\infty }rW\left( F_{\rho}\left( r\right) \right) d\rho \left( r\right)\text{. }
\]
{We require $W$ to be non-increasing and zero on $\left[ \zeta ,1\right] $ for some critical mass $\zeta <1$. The parameter $\zeta $ must be chosen on the basis of an estimate of the amount $\lambda $ of perturbing data.} Note that if $W$ is identically $1$ then $\Phi _{1} =\Phi$ in \eqref{eq:functional}, while if $W=\zeta ^{-1}1_{\left[ 0,\zeta \right] }$ then $\Phi _{W}$ is the hard thresholding functional in \eqref{eq:hard}.

We now propose to ``robustify" unsupervised learning by replacing the original problem (\ref{Unsupervised objective}) by
\begin{equation}
\min_{S\in \mathcal{S}}\Phi _{W}\left( \mu _{S}\right),
\label{Robused unsupervised learning problem}
\end{equation}
and denote a global minimizer by $S^{\dagger }\equiv S^{\dagger }\left( \mu\right)$. 
			
{In practice, $\mu$ is unknown and the} %
search for the model $S^{\dagger }$ has to rely on finite data. {If $\hat{\mu}\left( \mathbf{X}\right) = \frac{1}{n}\sum_{i=1}^{n}\delta _{X_{i}}$ is the empirical measure induced by an i.i.d. sample $\mathbf{X}=\left( X_{1},...,X_{n}\right) \sim \mu ^{n}$, }
then the empirical objective is the plug-in estimate
\begin{align}
& \Phi _{W}\left( \hat{\mu}\left( \mathbf{X}\right) _{S}\right)  =\frac{1}{n}\sum_{i=1}^{n}d_{S}\left( X_{i}\right) W\left( \frac{1}{n}\left\vert \left\{ X_{j}:d_{S}\left( X_{j}\right) \leq d_{S}\left(X_{i}\right) \right\} \right\vert \right) =\frac{1}{n}\sum_{i=1}^{n}d_{S}\left( X\right) _{\left( i\right) }W\left(\frac{i}{n}\right),
\label{eq:obj_alt}
\end{align}
where $d_{S}\left( X\right) _{\left( i\right) }$ is the $i$-th smallest member of $\left\{ d_{S}\left( X_{1}\right) ,...,d_{S}\left( X_{n}\right)\right\} $.

{The empirical estimate $\Phi _{W}\left( \hat{\mu}\left( \mathbf{X}\right)_{S}\right) $ is an $L$-statistic \cite{Serfling1980}. We denote a minimizer of this objective}
by 
\begin{equation}
\label{eq:RERM}
\hat{S}\left( \mathbf{X}\right) =\arg\min_{S\in \mathcal{S}}\Phi _{W}\left( \hat{\mu}\left( \mathbf{X}\right)_{S}\right).
\end{equation}
In the sequel we study three questions:
\begin{enumerate}
\item[{1}]If the underlying probability measure is a mixture $\mu =\left( 1-\lambda\right) \mu ^{\ast }+\lambda \nu $ of an unperturbed measure $\mu ^{\ast }$ and a perturbing measure $\nu $, and $S^{\dagger }=S^{\dagger }\left( \mu\right) $ is the minimizer of (\ref{Robused unsupervised learning problem}), under which assumptions will the reconstruction error $\Phi \left( \mu_{S^{\dagger }}^{\ast }\right) $ incurred by $S^{\dagger }$ on the unperturbed distribution approximate the minimal reconstruction error $\Phi\left( \mu _{S^{\ast }}^{\ast }\right) $?
\item[{2}] When solving (\ref{Robused unsupervised learning problem}) for a finite amount of data $\mathbf{X}$, under which conditions can we reproduce the behavior of $S^{\dagger }$ by the empirical  minimizer $\hat{S}\left( \mathbf{X}\right)$ in \eqref{eq:RERM}?
\item[{3}] How can the method be implemented and how does it perform in practice?
\end{enumerate}

\section{Resilience to Perturbations}
\label{sec:3.1}
Before we address the first question we make a preliminary observation in the tradition of robust statistics and compare the influence functions of the functional $\Phi _{1}$ to that one of the proposed $\Phi _{W}$ with bounded $W$, and $W\left( t\right) =0$ for $\zeta \leq t<1$. While we saw in (\ref{Influence non-robust}) that for any $\rho \in {\cal P}\left( \left[ 0,\infty \right) \right) $ the influence function ${\rm IF}\left( R;\rho ,\Phi \right) =R-\Phi \left( \rho\right) $ is unbounded in $R$, in the case of $\Phi_W$ we have, for any $R\in \mathbb{R}^{d}$, that

\begin{eqnarray*}
{\rm IF}\left( R;\rho ,\Phi _{W}\right) &\leq &{\rm IF}_{\max }\left( \rho ,W\right) \coloneqq \int_{0}^{F_{\rho }^{-1}\left( \zeta \right) }W\left( F_{\rho }\left(r\right) \right) F_{\rho }\left( r\right) dr.
\end{eqnarray*}
Notice that the right hand side is always bounded, which already indicates the improved robustness of $\Phi _{W}$ \cite{Hampel1974}. The upper bound ${\rm IF}_{\max }$ on the influence function plays also an important role in the subsequent analysis. %
			
Returning now to the data generating mixture $\mu =\left( 1-\lambda \right)\mu ^{\ast }+\lambda \nu $, where $\mu ^{\ast }\in \mathcal{P}\left( \mathbb{R}^{d}\right) $ is the the ideal, unperturbed distribution and $\nu \in\mathcal{P}\left( \mathbb{R}^{d}\right) $ the perturbation, we make the following assumption. 
			
\textbf{Assumption A}. There exists $S_{0}\in \mathcal{S}$, $\delta >0$, $\beta \in \left( 0,1-\lambda \right) $ and a scale parameter $r^{\ast } > 0$ (in units of squared euclidean distance), such that for	every model $S\in \mathcal{S}$ satisfying $\Phi \left( \mu _{S}^{\ast}\right) >\Phi \left( \mu _{S_{0}}^{\ast }\right) +\delta $ we have $F_{\mu_{S}}\left( r\right) <\beta F_{\mu _{S_{0}}^{\ast }}\left( r\right) $ for all $r\leq r^{\ast }$. 

Loosely speaking this assumption prescribes that, under the perturbed distribution $\mu$, any model $S$ with a large reconstruction error on $\mu^\ast$, should have its losses far less concentrated than the losses of $S_0$ around a small value (any $r \leq r^\ast$). As an example, on a typical sample from $\mu$ any such $S$ will have far more \emph{large} losses than $S_0$. For the sake of intuition, one should think of $S_0$ as $S^\ast(\mu^\ast)$ and $\beta$ as a very small number controlling the concentration of the losses. %
Equivalently, the assumption requires a perturbing distribution that is concentrated on no model $S$ very different from $S_0$ in terms of reconstruction error on the target. For concrete examples for the cases of \textsc{k-means} clustering and principal subspace analysis are given in  Figures \ref{fig:ex_mean} and \ref{fig:ex_psa}. %

We now state the main result of this section.

\begin{theorem}
\label{Theorem perturbed distribution}
Let $\mu ^{\ast },\nu \in \mathcal{P}\left( \mathbb{R}^{d}\right) $, $\mu =\left( 1-\lambda \right) \mu ^{\ast}+\lambda \nu $, and $\lambda \in \left( 0,1\right) $ and suppose there are $S_{0}$, $r^{\ast }$, $\delta >0$ and $\beta\in (0,1-\lambda)$, satisfying Assumption A. Let $W$ be nonzero on a set of positive Lebesgue measure, nonincreasing and $W\left( t\right) =0$ for $t\geq \zeta =F_{\mu_{S_{0}}}\left( r^{\ast }\right) $. Then ${\rm IF}_{\max }\left( \mu _{S_{0}},W\right) >0$, and if any $S\in\mathcal{S}$ satisfies

\begin{equation}
\Phi _{W}{(\mu_{S})} -\Phi _{W}{(\mu _{S_{0}})} \leq\left( 1{-}\frac{\beta }{1{-}\lambda }\right) {\rm IF}_{\max }{( \mu_{S_{0}},{W})}  
\label{Delta condition}
\end{equation}
then we have that $\Phi \left( \mu _{S}^{\ast }\right) \leq \Phi \left( \mu_{S_{0}}^{\ast }\right) +\delta $. In particular we always have that $\Phi \left(\mu _{S^{\dagger }}^{\ast }\right) \leq \Phi \left( \mu _{S_{0}}^{\ast}\right) +\delta $.  
\end{theorem}

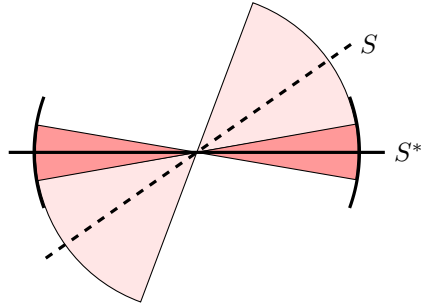
\begin{figure}[ht]
\begin{center}
\begin{tikzpicture}[scale=1.25]
\coordinate (c) at (0,0);
\filldraw[fill=red!10!white,draw=black] (0,0) -- (1.69cm,0.29cm) arc (10:70:1.7cm) -- cycle;
\filldraw[fill=red!10!white,draw=black] (0,0) -- (-1.69cm,-0.29cm) arc (190:250:1.7cm) -- cycle;
\filldraw[red!40!white,draw=black] (0,0) -- (1.69cm,-0.29cm) arc (-10:10:1.7cm) -- cycle;
\filldraw[red!40!white,draw=black] (0,0) -- (-1.69cm,0.29cm) arc (170:190:1.7cm) -- cycle;
\draw[very thick] (-2,0) -- (2,0);
\draw (2,0) node[anchor=west]{$S^{\ast}$};
\draw[dashed, very thick] (1.64,1.15) -- (-1.64,-1.15);
\draw (1.64,1.15) node[anchor=west]{$S$};
\draw [black,very thick,domain=-20:20] plot ({1.73 * cos(\x)}, {1.73 * sin(\x)});
\draw [black,very thick,domain=160:200] plot ({1.73 * cos(\x)}, {1.73 * sin(\x)});
\end{tikzpicture}
\end{center}
\caption{Illustration of Theorem \protect
\ref{Theorem perturbed distribution} for $d=2$ and $k=1$ in the case of \textsc{PSA}. The target distribution (dark gray) is concentrated on the subspace $S^*$, while the perturbing distribution (light gray) does not concentrate well on any individual subspace.}
\label{fig:ex_psa}
\end{figure}

We close this section by stating some important conclusions of the above theorem. 
\begin{enumerate}
\item A simplifying illustration of Theorem \ref{Theorem perturbed distribution} for principal subspace analysis is provided by Figure \ref{fig:ex_psa}. The distributions $\mu ^{\ast }$ and $\nu $ are assumed to have uniform densities $\rho \left( \mu ^{\ast }\right) $ and $\rho \left( \nu \right) $ supported on dark red and light red areas of the unit disk respectively. Suppose $\beta =\rho \left( \nu \right) /\rho \left( \mu ^{\ast }\right)<1-\lambda $, let $r^{\ast }=\sin ^{2}\left( \pi /\rho \left( \mu ^{\ast}\right) \right) $ and $\delta =4r^{\ast }$. If $\Phi \left( \mu _{S}^{\ast}\right) >\Phi \left( \mu _{S^{\ast }}^{\ast }\right) +\delta $ then the direction of the subspace $S$ does not intersect the black part of the unit circle and therefore $F_{\mu _{S}}\left( r\right) \leq \beta F_{\mu_{S^{\ast }}^{\ast }}(r)$ for all $r\leq r^{\ast }$. Thus Assumption A is satisfied and consequently, if $W\left( t\right) =0$ for $t\geq F_{\mu_{S^{\ast }}}\left( r^{\ast }\right) $, then $S^{\dagger }$ must intersect the black part of the unit circle and $\Phi \left( \mu _{S^{\dagger }}^{\ast}\right) \leq \Phi \left( \mu _{S^{\ast }}^{\ast }\right) +\delta $.
\item The generic application of this result assumes that $S_{0}=S^{\ast}\left( \mu ^{\ast }\right) $, but this is not required. Suppose $\mathcal{S}$ is the set of singletons and $\mu ^{\ast }$ is bimodal, say the mixture of distant standard normal distributions, and $\lambda =0$ for simplicity. Clearly there is no local concentration on the midpoint $S^{\ast }\left( \mu^{\ast }\right) $, but there is on each of the modes. If $S_{0}$ is the mean of the first mode and $\zeta $ is sufficiently small, then $S^{\dagger }$ can be near the mean of the other mode, because it has comparable reconstruction error. In this way the result also explains the astonishing behavior of our algorithm in clustering experiments with mis-specified number of clusters.
\item The conditions on $W$ prescribe an upper bound on the cutoff parameter $\zeta $. If the cutoff parameter $\zeta $ is chosen smaller (so that $W\left( t\right) =0$ for $t\geq \zeta \ll F_{\mu _{S^{\ast }}}\left( r^{\ast}\right) $), the required upper bound in (\ref{Delta condition}) decreases and it becomes more difficult to find $S$ satisfying the upper bound. This problem becomes even worse in practice, because the bounds on the estimation error also increase with $\zeta $, as we will see in the next section. 
\end{enumerate}

\section{Generalization Analysis}
\label{sec:5}

Up to this point we were working with distributions and essentially infinite data. In practice we only have samples $\mathbf{X}\sim \mu ^{n}$ and then it is important to understand to which extend we can obtain the conclusion of Theorem~\ref{Theorem perturbed distribution}, when $S$ is the minimizer of the empirical robust functional $\Phi _{W}\left( \hat{\mu}\left( \mathbf{X}\right)_{S} \right)$. 
This can be settled by a uniform bound on the estimation error for $\Phi _{W}$. 

\begin{proposition}
Under the conditions of Theorem \ref{Theorem perturbed distribution} with ${\bf X} \sim \mu^n$ we have that
\begin{multline*}
\Pr\left\{\Phi\left(\mu _{\hat{S}\left(\mathbf{X}\right) }^{\ast}\right)\leq\Phi\left(\mu _{S^{\ast }}^{\ast}\right) +\delta\right\} \geq \Pr \left\{ 2\sup_{S\in \mathcal{S}}\left\vert \Phi _{W}\left( \mu_{S}\right) -\Phi_{W}\left( \hat{\mu}_{S}\left( \mathbf{X}\right) \right)\right\vert \leq\left( 1-\frac{\beta }{1-\lambda } \right) {\rm IF}_{\max }\left( \mu_{S^{\ast }},W\right) \right\}.
\end{multline*}
\end{proposition}

The left hand side is the probability that the minimization of our robust $L$-statistic objective returns a $\delta$-optimal model for the target distribution $\mu^\ast$. The right hand side goes to $1$ as $n$ grows. As we show next, this is due to the fact that the class $\{\mu_S\}$ enjoys a uniform convergence property with respect to the functional $\Phi_{W}$. Particularly, we present three uniform bounds that control the rate of decay of the same estimation error $\left\vert \Phi _{W}\left( \mu_{S}\right) -\Phi_{W}\left( \hat{\mu}_{S}\left( \mathbf{X}\right) \right)\right\vert$.

The first two bounds are dimension-free and rely on Rademacher and Gaussian averages of the function class $\left\{ x\mapsto d\left( x,S\right) :S\in \mathcal{S}\right\} $. Bounds for these complexity measures in the practical cases considered can be found in \cite{Maurer2010}. Our last bound is dimension dependent but may outperform the other two if the variance of the robust objective is small under its minimizer. All three bounds require special properties of the weight function $W$. 

For this section we assume $\mu \in \mathcal{P}\left( \mathbb{R}^{d}\right)$ to have compact support, write $\mathcal{X}=$support$\left( \mu \right)$ and let $\mathcal{F}$ be the function class 
\[
\mathcal{F}=\left\{ x\in \mathcal{X}\mapsto d\left( x,S\right) :S\in\mathcal{S}\right\} .
\]
We also set $R_{\max }=\sup_{f\in \mathcal{F}}\left\Vert f\right\Vert_{\infty }$.

The first bound is tailored to the hard-threshold %
$\zeta ^{-1}1_{\left[0,\zeta \right] }$. It follows directly from the elegant recent results of \cite{Lee2020}. For the benefit of the reader we give a proof in the appendix, without any claim of originality and only slightly improved constants.

\begin{theorem}
\label{Theorem hard threshold uniform bound} Let $W=\zeta ^{-1}1_{\left[ 0,\zeta \right] }$ and $\eta >0$. With probability at least $1-\eta $ in $\mathbf{X}\sim \mu ^{n}$ we have that 
\begin{equation*}
\sup_{S\in \mathcal{S}}\left\vert \Phi _{W}\left( \mu _{S}\right) -\Phi_{W}\left(\hat{\mu}_{S}\left(\mathbf{X}\right) \right) \right\vert \leq \frac{2}{\zeta n}\mathbb{E}_{\mathbf{X}}\mathcal{R}\left( \mathcal{F}, \mathbf{X}\right) +\frac{R_{\max }}{\zeta \sqrt{n}}\left( 2+\sqrt{\frac{\ln \left( 2/\eta \right) }{2}}\right) ,    
\end{equation*}
where $\mathcal{R}\left( \mathcal{F},\mathbf{X}\right) $ is the Rademacher average 
\[
\mathcal{R}\left( \mathcal{F},\mathbf{X}\right) =\mathbb{E}_{\mathbf{\epsilon }}\left[ \sup_{S\in \mathcal{S}}\sum_{i=1}^{n}\epsilon _{i}d\left(X_{i},S\right) \right] 
\]
with independent Rademacher variables $\mathbf{\epsilon }=\left( \epsilon_{1},...,\epsilon _{n}\right) $.
\end{theorem}

The next bound requires boundedness and a Lipschitz property for the weight function $W$ which can otherwise be arbitrary. We define the norm $\left\Vert W\right\Vert _{\infty }=\sup_{t\in \left[ 0,1\right] }\left\vert W\left(t\right)\right\vert $ and seminorm $\left\Vert W\right\Vert _{\rm Lip}=\inf \left\{L:\forall t,s\in \left[ 0,1\right] ,~W\left( t\right) -W\left( s\right) \leq
L\left\vert t-s\right\vert \right\} .$

\begin{theorem}
\label{Theorem Lipschitz uniform bound}
For any $\eta >0$
\begin{equation*}
\sup_{S\in \mathcal{S}}\left\vert \Phi _{W}\left( \mu _{S}\right) -\Phi_{W}\left( \hat{\mu}_{S}\left( \mathbf{X}\right) \right) \right\vert \leq \frac{2\sqrt{\pi }\left( R_{\max }\left\Vert W\right\Vert _{\infty}+\left\Vert W\right\Vert _{\rm Lip}\right) }{n}~\mathbb{E}_{\mathbf{X}}\mathcal{G}\left(\mathcal{F},\mathbf{X}\right) \\ +R_{\max }\left\Vert W\right\Vert _{\infty }\sqrt{\frac{2\ln \left( 2/\eta\right) }{n}}    
\end{equation*}
where $\mathcal{G}\left( \mathcal{F},\mathbf{X}\right) $ is the Gaussian average
\[
\mathcal{G}\left( \mathcal{F},\mathbf{X}\right) =\mathbb{E}_{\mathbf{\gamma }}\left[ \sup_{S\in \mathcal{S}}\sum_{i=1}^{n}\gamma _{i}d\left(X_{i},S\right) \right] ,
\]
with independent standard normal variables $\gamma _{1},...,\gamma _{n}$.
\end{theorem}
Our last result also requires a Lipschitz property for $W$ and uses a
classical counting argument with covering numbers for a variance-dependent
bound.
\begin{theorem}
\label{Theorem Variance uniform bound}
Under the conditions of the previous theorem, with probability at least $1-\eta $ in $\mathbf{X}\sim \mu ^{n}$ we have that for all $S\in \mathcal{S}$%
\begin{equation*}
\left\vert \Phi _{W}\left( \mu _{S}\right) -\Phi _{W}\left( \hat{\mu}_{S}\left(\mathbf{X}\right) \right) \right\vert \leq \sqrt{2V_{S}C}+\frac{6R_{\max }\left( \left\Vert W\right\Vert_{\infty}+\left\Vert W\right\Vert _{\rm Lip}\right) C}{n} +\frac{\left\Vert W\right\Vert _{\infty }R_{\max }}{\sqrt{n}},
\end{equation*}
where $V_{S}$ is the variance of the random variable $\Phi _{W}\left( \hat{\mu}_{S}\left( \mathbf{X}\right) \right) $, and $C$ is the complexity term
\[
C=kd\ln \left( 16n\left\Vert \mathcal{S}\right\Vert ^{2}/\eta \right) 
\]
if $\mathcal{S}$ is the set of sets with $k$ elements, or convex polytopes with $k$ vertices and $\left\Vert \mathcal{S}\right\Vert =\sup_{x\in S\in\mathcal{S}}\left\Vert x\right\Vert $, or
\[
C=kd\ln \left( 16nR_{\max }^{2}/\eta \right) 
\]
if $\mathcal{S}$ is the set of set of $k$-dimensional subspaces.
\end{theorem}

We state two important conclusion from the above theorems.
\begin{enumerate}
\item Our bounds decrease at least as quickly as $n^{-1/2}\ln n$. However, the bound in the last theorem may be considerably smaller than the previous two if $n$ is large and the unperturbed distribution is very concentrated. The last term, which is of order $n^{-1/2}$ does not carry the burden of the complexity measure and decays quickly. The second term contains the complexity, but it decreases as $n^{-1}$. It can be shown from the Efron-Stein inequality (see e.g.\cite{boucheron2013} Theorem 3.1) that the variance $V_{S}$ of our $L$-statistic estimator is at most of order $n^{-1}$, so the entire bound is at most of order $n^{-1/2}\ln n$. On the other hand $V_{s}$ can be very small. For example, if the unperturbed distribution is completely concentrated at $S^{\ast }$ and $\zeta $ is chosen appropriately $V_{S^{\ast }}=0$ and, apart from the complexity-free last term the decay is as $n^{-1}\ln n$.

\item The above bounds implies that, by equating the estimation error to $\frac{1}{2}\big(1{-}\frac{\beta}{1{-}\lambda}\big){\rm IF}_{\max }\left( \mu_{S^{\ast }},W\right)$ and solving for $\eta$, our method recovers a $\delta$-optimal (w.r.t. $\mu^\ast$) model with probability at least equal to $1-\exp(-n)$.
\end{enumerate}
Finally, we highlight that the above uniform bounds may be of independent interest. For example, consider the case that the test data also come from the perturbed distribution. In such a situation one might be interested in evaluating the performance of the learned model only on data that fit the model class, i.e. $\Phi_W{\mu_S}$. These bounds guarantee that by minimizing the empirical robust functional, one also get good performances on future data from the same distribution.

\section{Algorithms}
\label{sec:4}
In this section we present our algorithm for (approximately) minimizing the robust $L$-statistic $\Phi_W(\hat{\mu}({\bf X})_S)$ w.r.t. model $S \in \Models$. Throughout 
we assume $W$ non-increasing and fixed, and to simplify the notation we use the shorthand  ${\hat \Phi}_S({\bf X}) \equiv \Phi_W(\hat{\mu}({\bf X})_S)$.

\subsection{General Algorithm}
Let ${\bf x} = (x_1,\ldots,x_n)$ be a realization of $\mathbf{X} \sim \mu^n$, consider the following function of $S \in \Models$  
\begin{equation}
{\hat \Phi}_S({\bf x}) = \frac{1}{n} \sum_{i=1}^n W\left(\frac{\pi(i)}{n}\right) d_S(x_i) 
\label{eq:robust_obj}
\end{equation}
where $\pi$ is the ascending ordering of the $d_S(x)_{(i)}$ and notice that minimizing \eqref{eq:robust_obj} is equivalent to minimize \eqref{eq:obj_alt}. Let $p$ any fixed element in ${\rm Sym}_n$\footnote{Here ${\rm Sym}_n$ denotes the set of all $n!$ permutations over $n$ objects.} and let 
\[
\phi_S({\bf x}, p) = \frac{1}{n} \sum_{i=1}^n W\left( \frac{p(i)}{n}\right) d_S(x_i).
\]
In the following we will leverage the following property of $\phi_S$.
\begin{lemma}
\label{lemma:optimality_lemma}
For any $S \in \Models$ and any $p \in {\rm Sym}_n$, if $\pi$ is the ascending ordering of the $d_S(x_i)$s, then $\phi_S({\bf x}, p) \ge \phi_S({\bf x}, \pi) = {\hat \Phi}_S({\bf x})$.
\end{lemma}
We need also the following definition.
\begin{definition}
\label{def:descent_oracle}
A mapping $\DO: \Models \times S_n \rightarrow \Models$ is a Descent Oracle for $\phi_S$ iff for any $S \in \Models$ and any $p \in {\rm Sym}_n$, $\phi_{\DO(S, p)}({\bf x}, p) \le \phi_S({\bf x}, p)$.
\end{definition}
The algorithm attempts to minimize \eqref{eq:robust_obj} via alternating minimization of $\phi_S$. At the beginning, it picks an initial model $S_0$ and sort the induced losses in ascending order, i.e. pick the optimal permutation $\pi_0$. Then it starts iterating this two steps by first calling the descent oracle $\DO(S_t, \pi_t)$ and then sorting the induced losses. At each step either the permutation $\pi_t$ or the model $S_t$ are fixed. Pseudocode is given in \Cref{alg:RUL}.
\begin{algorithm}[t!]
\caption{}
\label{alg:RUL}
\begin{algorithmic}[1]
{
\STATE Pick any $S_0 \in \Models$
\STATE $\pi_0 \gets \argmin_{p \in {\rm Sym}_n} \phi_{S_0}({\bf x}, p)$
\FOR{$t = 1,\ldots, T$} 
\STATE $S_t \gets \DO(S_{t-1}, \pi_{t-1})$
\STATE $\pi_t \gets \argmin_{p \in S_n} \phi_{S_t}({\bf x}, p)$
\ENDFOR 
\STATE \textbf{return} $S_T$
}
\end{algorithmic}
\end{algorithm}
Indeed, at each step the algorithm first finds a descending iteration $S_{t+1}$ of $\phi_{S_t} ({\bf x}, \pi_t)$ and then sort the losses according to $\pi_{t+1}$, an operation that by \Cref{lemma:optimality_lemma} cannot increase the value of $\phi_{S_{t+1}}$. Thus the following holds.
\begin{theorem}
\Cref{alg:RUL} is a descent algorithm for the problem of minimizing \eqref{eq:robust_obj}, i.e. for any $t, {\hat \Phi}_{S_{t+1}}({\bf{x}}) \le {\hat \Phi}_{S_t}({\bf x})$.
\end{theorem}
This algorithm is general and to apply it to a specific learning problem an implementation of the descent oracle is needed. The efficiency of \Cref{alg:RUL} depends upon such oracle. In the following we show two descent oracles for the cases of \textsc{kmeans} and \textsc{psa}. We complement these results with a computational lower bound showing that, in general, minimizing \eqref{eq:robust_obj} is NP-Hard.

\newcommand{\figw}{1.7in}
\newcommand{\figh}{1.37in}
\begin{figure*}[t!]
\centering
\begin{subfigure}
\centering
\includegraphics[width=\figw, height=\figh]{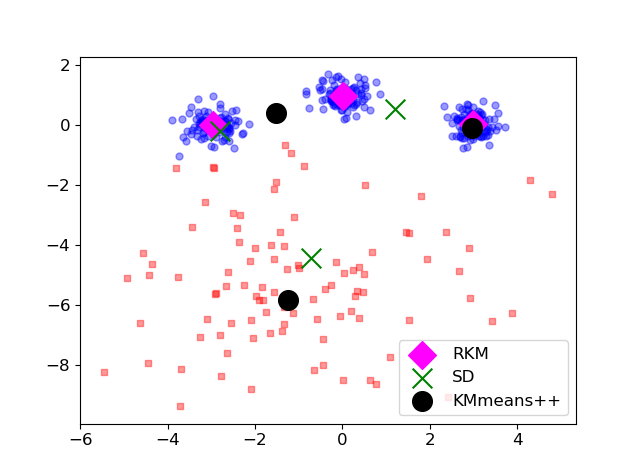}
\end{subfigure}
\begin{subfigure}
\centering
\hspace{-.5truecm}\includegraphics[width=\figw, height=\figh]{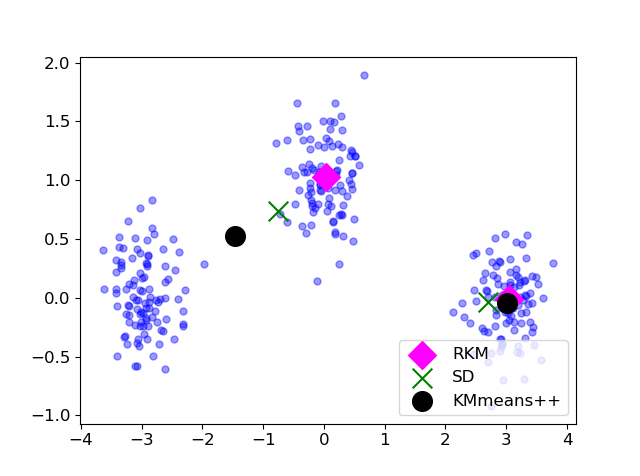}
\end{subfigure}
\begin{subfigure}
\centering
\hspace{-.3truecm}\includegraphics[width=\figw, height=\figh]{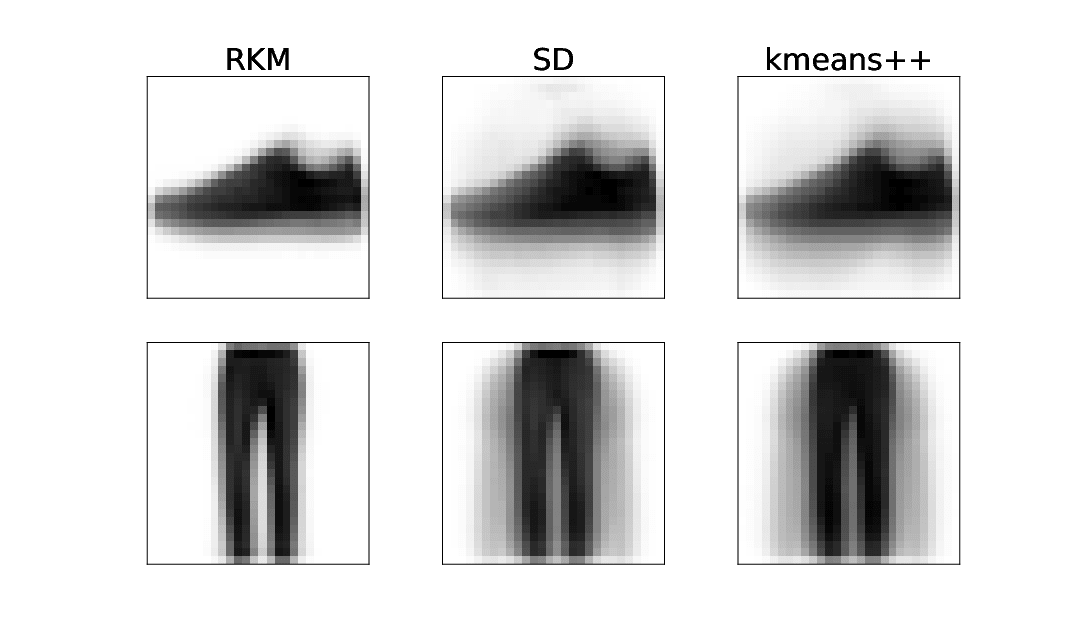}
\end{subfigure}
\begin{subfigure}
\centering
7\hspace{-.5truecm}\includegraphics[width=\figw, height=\figh]{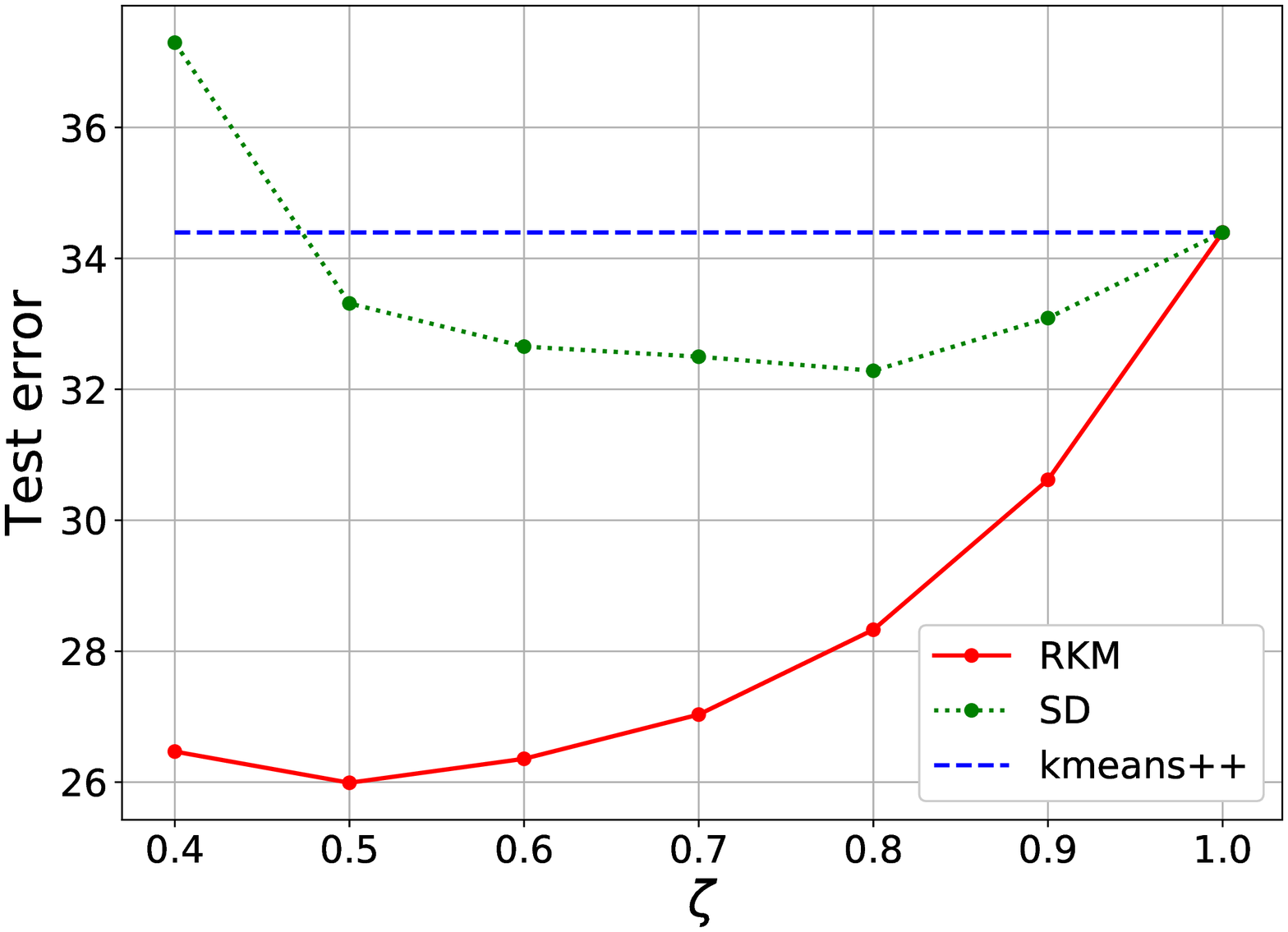}
\end{subfigure}
\vskip -0.2in
\caption{Experiments for \textsc{kmeans} on synthetic data and real data with $k=2$ and $k=3$.}
\label{fig:km}
\end{figure*}
\begin{theorem}
\label{th:hardness}
Minimizing \eqref{eq:robust_obj} for the case of \textsc{kmeans} when $k = 1$ and $W$ is the hard threshold is NP-Hard.
\end{theorem}
Notice that in the case of \textsc{kmeans} when $W$ is the identity, the problem reduces to finding the optimal \textsc{kmeans} solution, a problem which is known to be hard. However, \textsc{kmeans} admits a simple closed form solution when $k = 1$; in some sense minimizing the robust objective is even harder than standard \textsc{kmeans}. The immediate consequence of this result is that approximate solutions to the problem of minimizing \eqref{eq:robust_obj} are the best one can get; our algorithms, are a first step towards the design of methods with provable approximation guarantees.

\paragraph{$k$-Means Clustering ({\textsc KMEANS}).} In this case $\Models$ is the set of all possible $k$-tuples of centers in $\R^d$ and $d_S(x) = \min_{c \in S} \|x - c\|_2^2$. Keeping fixed the permutation $p$, we consider as descent oracle the following Lloyd-like update for the centers. Each center $c \in S$ induces a cluster formed by a subset of training points $x_i$, $i \in {\cal I}$ which are closer to $c$ than every other center (breaking ties arbitrarily). The overall \emph{loss} of representing point in ${\cal I}$ with $c$ is
\[
\sum_{i \in {\cal I}} W\left(\frac{p(i)}{n}\right) \|x_i - c \|_2^2.
\]
This loss is minimized at 
\[{\hat c} = \frac{1}{\sum_{i \in {\cal I}}W\left(\frac{p(i)}{n}\right)} \sum_{i \in {\cal I}} W\left(\frac{p(i)}{n}\right) x_i,
\] 
so the following holds.

\begin{proposition}
Given $S$ and $p$, the mapping that for every $c \in S$ returns the ${\hat c}$ defined above is a descent oracle for \textsc{kmeans} and its runtime is $O(nkd)$.
\end{proposition} 

The resulting algorithm can is a generalization of the method proposed in \cite{Chawla2013}.

\paragraph{Principal Subspace Analysis (\textsc{psa}).} In this case $\Models$ is the set of all possible $d \times k$ matrices $U$ such that $U^\top U = I_d$, $d_S(x) = \|x - UU^\top x\|_2^2$ and
\[
\phi_{U}({\bf x}, p) = \sum_{i=1}^n W\left(\frac{p(i)}{n}\right)
 \|x - UU^\top x_i\|_2^2.
\]
Given $p$, it is easy to see that the above function is minimized at the matrix ${\hat U}$ formed by stacking as columns the $k$ eigenvectors of $\sum_{i}^n W\left(\frac{p(i)}{n}\right) x_ix_i^\top $ associated to the top $k$ eigenvalues, so the following holds.

\begin{proposition}
Given $U$ and $p$, the mapping that returns the ${\hat U}$ defined above is a descent oracle for  \textsc{PSA} and its runtime is $O(\min\{d^3 + nd^2, n^3 + n^2d\})$.
\end{proposition}

\section{Experiments}
\label{sec:exps}
The purpose of the numerical experiments is to show that:

\begin{itemize}
\item Our algorithms for \textsc{psa} and \textsc{kmeans} outperform standard SVD, \textsc{kmeans++} and the \emph{Spherical Depth} method (SD) in presence of outliers, while obtain similar performances on clean data. 

\item Our algorithms on real data are not too sensitive to the parameters of the weight function. In particular, we show that there exist a wide-range of $\zeta$ values such that using the hard-threshold function leads to good results.

\item In the case of \textsc{kmeans} our method is able to accurately reconstruct some of the true centers even when the value of $k$ is miss-specified. This matches the second remark after Theorem 1.
\end{itemize}

\paragraph{Implemented Algorithms.} For \textsc{kmeans}++ we used the {\em sklearn} implementation fed with the same parameters for the maximum number of iterations $T$ and the initializations $r$ we used for our method. Notice that $T$ is only an upper bound to the number of iterations, the algorithms stop when the difference between the current objective value and the previous one is smaller than $10^{-7}$. To set $r$ we used the largest value before diminishing returns were observed. For standard PSA we compute the SVD of $\sum_{i} x_ix_i^\top$. The SD method is a general purpose pre-processing technique that is applied on the data before performing \textsc{kmeans} and \textsc{PSA} (see e.g. \cite{Elmore2006, Fraiman2019}). This method computes a score for each point in the dataset by counting in how many balls, whose antipodes are pairs of points in the data, it is contained. The $1-\zeta n$ points with the smallest scores are discarded. If the data contain $n$ points, the methods needs to check $O(n^2)$ balls for each of the $n$ point resulting in a runtime of $O(n^3)$. For scalability on real data, we implemented a randomized version of this method that for each point only check $M$ balls picked uniformly at random from the set of all possible balls and used $M = O(n)$; the resulting runtime is $O(n^2)$. In the following we refers to our methods as RKM and RPSA respectively. All experiments have been run on an standard laptop equipped with an Intel i9 with 8 cores each working at 2,4 GHz and 16 GB of RAM DDR4 working at 2,6 GHz.
\subsection{{\textsc KMEANS} Clustering}
\paragraph{Synthetic Data.} We run two experiments with artificial data in $\R^2$. In the first experiment, we generated 300 inliers from 3 isotropic truncated Gaussians (100 points each) with variance $0.1$ along both axis and mean $(-3, 0)$, $(0, 1)$ and $(3, 0)$ respectively. We then corrupt the data adding 100 points from a fourth isotropic truncated Gaussian centered at $(-1, -5)$ with variance $5$ along both axis. For both RKM and \textsc{kmeans++} we $T = 10$ and $r = 30$. We initialized \RKM\ with uniform centers and set $\zeta = 0.75$, the same $\zeta$ is used for SD. Results are shown in \Cref{fig:km} top left, where it is possible to see that while \RKM\ recovers the true centers, SD and \textsc{kmeans}++ both fail badly placing one centers in the middle of the two clusters and the other close to the mean of the perturbing distribution. In the second experiment, we generated 300 points from the same 3 inliers Gaussians and set the algorithms with $k = 2$ and $\zeta=0.6$, while $T$ and $r$ are as above. Results are shown in the top right of \Cref{fig:km}, where it is possible to see that \textsc{kmeans}++ and SD -- although to a lesser extend -- wasted a center to merge 2 clusters, while RKM correctly recovers 2 out of the 3 centers.

\begin{figure}[t!]
\centering
\includegraphics[width=1.76in, height=1.37in]{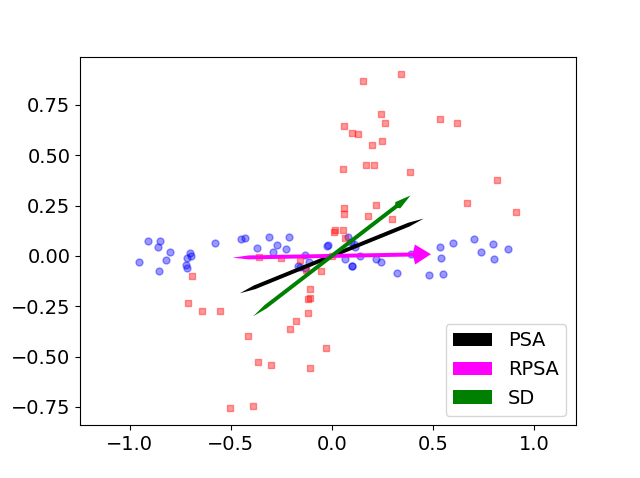}
\includegraphics[width=1.76in, height=1.37in]{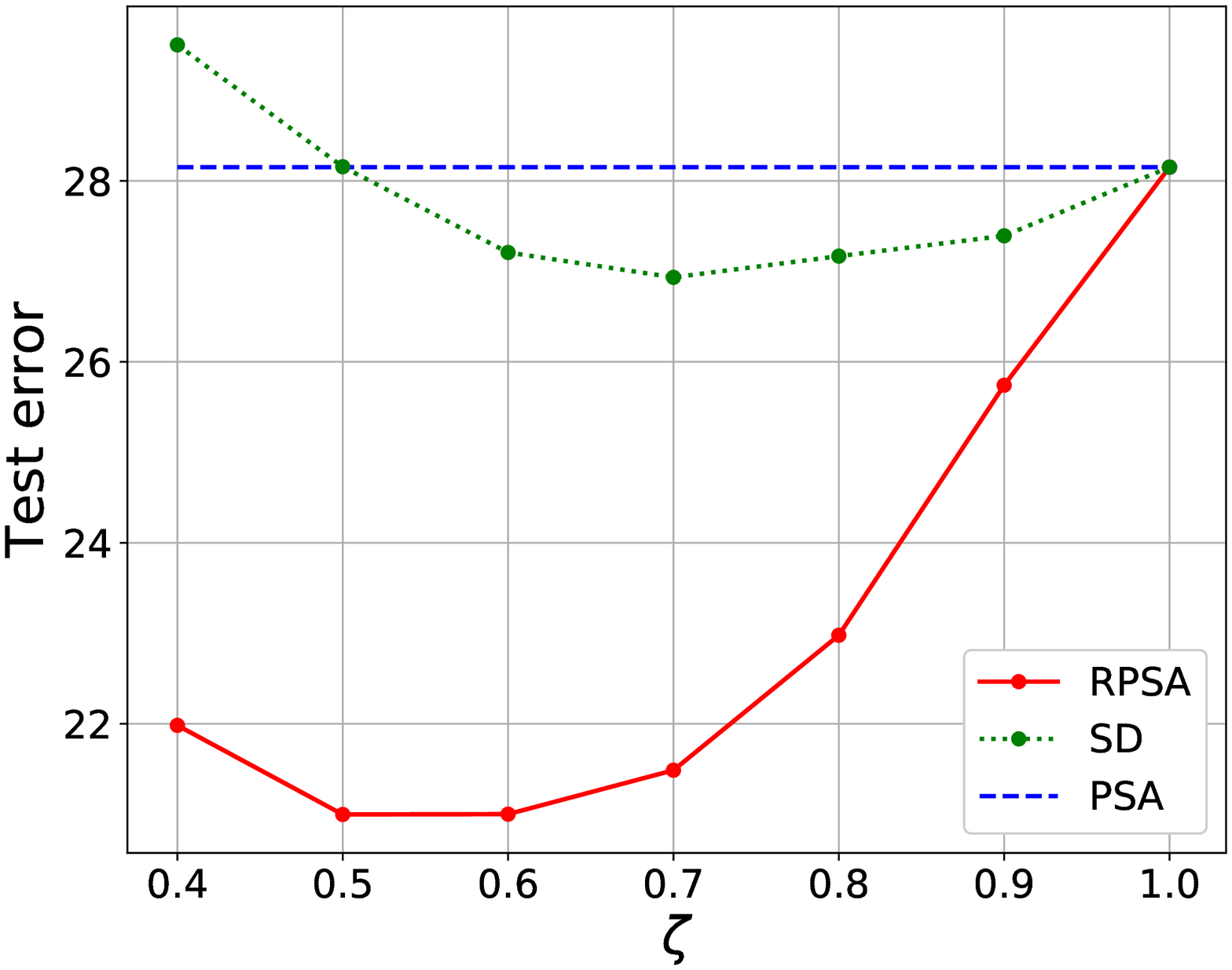}
\caption{Experiments for \textsc{PSA} on synthetic data and real data with $k=1$ and $k=2$.}
\label{fig:psa}
\end{figure}
\paragraph{Real Data.} %

In the synthetic experiments we choose $\zeta$ according to the exact fraction of outliers, a quantity which is usually unknown in practice. Here we show that there is a wide range of values for $\zeta$ such that RKM performs better than \textsc{kmeans}++. We used the Fashion-MNIST dataset which consists of about 70000 $28\times28$ images of various types of clothes splitted in a training set of 60000 images and a test set of 10000 images. Specifically, there are 10 classes in the dataset: t-shirts, trousers, pullover, dresses, coats, sandals, shirts, sneakers, bags and ankle boots. The training data were generated by sampling 1000 points, from the training set, each from the sneakers and the trousers classes as inliears, and 250 points from each other class as outliers.  The resulting fraction of outliers is about 0.5.  The test data consist of all the sneakers and the trousers in the test set and has size of about 2000.  We run the algorithms with $T = 50, r = 30, M = 4000, k = 2$ and $\zeta$ in the range $[0.4,1]$. Results are shown in the bottom row of \Cref{fig:km}. In the lower left, it is possible to see that the centers learned by RKM at the optimal threshold value $\zeta = 0.5$ look good, while the centers found by SD and \textsc{kmeans}++ are affected by the outliers. Specifically, the such centers arise from the overlap of multiple classes. One center suffers from the effect of the other two shoes classes (sandald and boots) as witnessed by the elongated background area, while the other is affected by the clothes classes (most noticeably, the coats) as suggested by background shadow. As for the reconstruction error, RKM outperforms SD uniformly over the range of considered values of $\zeta$.

\subsection{Principal Subspace Analysis}

\paragraph{Synthetic Data.} We run a synthetic experiment with artificial data in $\R^2$. We generate 50 points from the uniform distribution over $[-1, 1] \times [-0.1, 0.1]$ as inliers and 50 points for the uniform distribution over $\R_{++} \cup \R_{--} \cap B(0, 1)$\footnote{Here with $\R_{++}$ and $\R_{--}$ we denote the top right and the bottom left orthant of $\R^2$.} as outliers. We run RPSA with $T = 50$, $r = 30$, $\zeta = 0.5$ and initialize $U$ as a normalized Gaussian matrix. We set $k = 1$ for all algorithms. Results are shown in the left plot of \Cref{fig:psa} where it is possible to see that the principal subspace learned by RPSA is not affected by the outliers, as opposed to SD and PSA.

\paragraph{Real Data.} %

Similarly to the case of \textsc{kmeans}, we tested our method on real data for a range of values of $\zeta$. We used again the same setting as before on the Fashion-MNIST dataset. We run the algorithms we $T=50$, $r=5$, $M = 4000$, $k=2$ and $\zeta$ in the range $[0.4, 1]$. Results are shown in the right plot of \Cref{fig:psa}, where it is possible our algorithm outperforms both PSA and does better than SD.

\section{Conclusions and Future Works}
In this work, we address the important problem of designing robust methods for unsupervised learning. We proposed a novel general framework, based on the minimization of an $L$-statistic, to design algorithms that are resilient to the presence of outliers and/or to model miss-specification. Our method has strong statistical guarantees, is flexible enough to incorporate many problems in unsupervised learning and is effective in practice as the experiments reveal. On the other hand, several extensions can be considered. First, here we studied in details \textsc{kmeans} and \textsc{psa}, but our theory also covers the cases of \textsc{kmedian}, sparse coding or non-negative matrix factorization. A related improvement also regards the design of methods for the choice of $\zeta$ which do not require an estimate of the fraction of outliers. Second, we believe that this framework can be extended to supervised learning problems such us canonical correlation analysis and partial least squares. Third, our algorithm has only a descent property, and it would be interesting to design algorithms with stronger guarantees such as provable approximation properties.

\bibliographystyle{plain}
\bibliography{icml2021}

\newpage
\onecolumn

\appendix
\section*{Supplementary Material}

The supplementary material is organized as follows: %
\begin{itemize}
    \item 
In Appendix~\ref{app:A} we prove the statistical properties of the proposed method; in particular we prove Theorems 1, 3 and 5. 
\item In Appendix \ref{app:B} we give a proof of the hardness result described by Theorem 10. 
\item Finally, in Appendix \ref{app:C} we present additional experiments with the proposed method for the case of \textsc{k-means}.
\end{itemize}

\section{Statistical Properties of the Proposed Method}
\label{app:A}
We first analyze some basic properties of the functional $\Phi _{W}$. The following is easily seen to be an alternative definition of $\Phi _{W}$.
\[
K_{W}\left( t\right)  =\int_{0}^{t}W\left( u\right) du
\]
and\[
\Phi _{W}\left( \rho \right)  =\int_{0}^{\infty }rdK_{W}\left( F_{\rho
}\left( r\right) \right) \text{ for }\rho \in \mathcal{P}\left( \left[
0,\infty \right) \right) \text{.}
\]
From this we find 

\begin{lemma}
\label{Lemma derivative} For $\rho _{1},\rho _{2}\in \mathcal{P}$ and $W$ bounded
\begin{equation}
\Phi _{W}\left( \rho _{1}\right) -\Phi _{W}\left( \rho _{2}\right)
=-\int_{0}^{\infty }\left( K_{W}\left( F_{\rho _{1}}\left( r\right) \right)
-K_{W}\left( F_{\rho _{2}}\left( r\right) \right) \right) dr,
\label{Difference formula}
\end{equation}
and
\[
\frac{d}{dt}\Phi _{W}\left( \left( 1-t\right) \rho _{1}+t\rho _{2}\right)
=\int_{0}^{\infty }W\left( F_{\left( 1-t\right) \rho _{1}+t\rho _{2}}\left(
r\right) \right) \left( F_{\rho _{1}}\left( r\right) -F_{\rho _{2}}\left(
r\right) \right) dr.
\]
\end{lemma}

\begin{proof}
Since members of $\mathcal{P}$ have finite first moments we have for any $\rho \in \mathcal{P}$ that $r\rho \left( r,\infty \right) \rightarrow 0$ as $r\rightarrow \infty $, so
\[
\lim_{r\rightarrow \infty }r\left( K_{W}\left( F_{\rho _{1}}\left( r\right)
\right) -K_{W}\left( F_{\rho _{2}}\left( r\right) \right) \right) \leq
\left\Vert W\right\Vert _{\infty }\lim_{r\rightarrow \infty }r\left\vert
\rho _{2}\left( r,\infty \right) -\rho _{1}\left( r,\infty \right)
\right\vert =0,
\]
and the formula (\ref{Difference formula}) follows from integration by parts. Thus for arbitrary $\rho \in \mathcal{P}$
\[
\Phi _{W}\left( \left( 1-t\right) \rho _{1}+t\rho _{2}\right) -\Phi
_{W}\left( \rho \right) =-\int_{0}^{\infty }\left( K_{W}\left( \left(
1-t\right) F_{\rho _{1}}\left( r\right) +tF_{\rho _{2}}\left( r\right)
\right) -K_{W}\left( F_{\rho }\left( r\right) \right) \right) dr.
\]
Taking the derivative w.r.t. $t$ and using the chain rule and $K_{W}^{\prime}=W$ gives the second identity. 
\end{proof}

We now analyze the influence function of the functional $\Phi _{W}$.

\begin{lemma}
\label{Lemma Influence} Let $R\in \left[ 0,\infty \right) $, $\rho \in \mathcal{P}\left( \left[ 0,\infty \right) \right)$
	
(i) If $W$ is nonnegative, bounded and $W\left( t\right) =0$ for $t\geq \zeta $ and $\zeta <1$ then 

\[
IF\left( R;\rho ,\Phi _{W}\right) \leq IF_{\max }\left( \rho ,W\right)
:=\int_{0}^{F_{\rho }^{-1}\left( \zeta \right) }W\left( F_{\rho }\left(
r\right) \right) F_{\rho }\left( r\right) dr.
\]
	
(ii) If $\zeta >0$, $W=\zeta ^{-1}1_{\left[ 0,\zeta \right] }$, $\rho $ is non-atomic and 
$F_{\rho }^{-1}\left( \zeta \right) \left( \rho \right) =F_{\rho }^{-1}\left( \zeta \right) =F_{\rho }^{-1}\left( \zeta \right)$.
Then
\begin{eqnarray*}
IC\left( R;\rho ,\Phi _{W}\right)  &=&\left\{ 
\begin{array}{ccc}
\zeta ^{-1}\left( R+\left( \zeta -1\right) F_{\rho }^{-1}\left( \zeta\right) \right) -\Phi _{W}\left( \rho \right)  & ~{\rm if}~ & 0\leq R\leq F_{\rho }^{-1}\left( \zeta \right)  \\ 
F_{\rho }^{-1}\left( \zeta \right) -\Phi _{W}\left( \rho \right)  & ~{\rm if}~
& F_{\rho }^{-1}\left( \zeta \right) <R%
\end{array}
\right.  \\
&\leq &F_{\rho }^{-1}\left( \zeta \right) -\Phi _{W}\left( \rho \right) .
\end{eqnarray*}%
 
\end{lemma}

\begin{proof}
(i) In the second conclusion of Lemma \ref{Lemma derivative}, letting $\rho_{2}=\delta _{R}$ and taking the limit $t\rightarrow 0$ we obtain the influence function
\[
IF\left( R;\rho ,\Phi _{W}\right) =\int_{0}^{\infty }W\left( F_{\rho }\left(
r\right) \right) \left( F_{\rho }\left( r\right) -F_{\delta _{R}}\left(
r\right) \right) dr.
\]
Part (i)\ follows.

(ii) From Lemma \ref{Lemma derivative} we get%
\begin{eqnarray*}
\frac{d}{dt}\Phi \left( \left( 1-t\right) \rho +t\delta _{R}\right) \left(t=0\right)  &=&\zeta ^{-1}\int_{0}^{F_{\rho }^{-1}\left( \zeta \right)	}\left( F_{\rho }\left( r\right) -1_{\left[ R,\infty \right) }\left(r\right) \right) dr \\
&=&\zeta ^{-1}\left( \int_{0}^{F_{\rho }^{-1}\left( \zeta \right) }F_{\rho}\left( r\right) dr-\int_{0}^{F_{\rho }^{-1}\left( \zeta \right) }1_{\left[		R,\infty \right) }\left( r\right) dr\right).
\end{eqnarray*}
From integration by parts the first term in parenthesis is $\zeta \left(F_{\rho }^{-1}\left( \zeta \right) -\Phi _{W}\left( \rho \right) \right)$. The second term is zero if $F_{\rho }^{-1}\left( \zeta \right) <R$, otherwise it is $F_{\rho }^{-1}\left( \zeta \right) -R$. This gives the identity. For the inequality observe that $R\leq F_{\rho }^{-1}\left( \zeta\right) $ implies $\zeta ^{-1}\left( R+\left( \zeta -1\right) F_{\rho}^{-1}\left( \zeta \right) \right) \leq F_{\rho }^{-1}\left( \zeta \right)$. 
\end{proof}

\subsection{Resilience to Perturbations}
\label{app:A1}

We prove Theorem 1.

\begin{lemma}
\label{Lemma Robust}Let $S,S^{\ast }\in \mathcal{S}$, $\mu \in \mathcal{P}\left( \mathbb{R}^{d}\right) $, and suppose that there exists $r^{\ast }>0$ and $\alpha \in \left( 0,1\right) $ such that

\begin{equation}
\forall r\in \left( 0,r^{\ast }\right) ,\text{ }F_{\mu _{S}}\left( r\right)
\leq \alpha F_{\mu _{S^{\ast }}}\left( r\right) \text{.}  \label{Condition A}
\end{equation}
If $W$ is nonzero on a set of positive Lebesgue measure, nonincreasing and $W\left( t\right) =0$ for all $t\geq F_{\mu _{S^{\ast }}}\left( r^{\ast}\right) $ then

\[
\Phi _{W}\left( \mu _{S}\right) -\Phi _{W}\left( \mu _{S^{\ast }}\right)
\geq \left( 1-\alpha \right) \int_{0}^{r^{\ast }}W\left( F_{\mu _{S^{\ast}}}\left( r\right) \right) F_{\mu _{S^{\ast }}}\left( r\right) dr=\left(1-\alpha \right) IF_{\max }\left( \mu _{S^{\ast }},W\right) >0\text{.}
\]
\end{lemma}
		
\begin{proof}
By Lemma \ref{Lemma derivative} and the fundamental theorem of calculus

\[
\Phi _{W}\left( \mu _{S}\right) -\Phi _{W}\left( \mu _{S^{\ast }}\right)
=\int_{0}^{\infty }\left( \int_{\left[ 0,1\right] }W\left( sF_{\mu_{S}}\left( r\right) +\left( 1-s\right) F_{\mu _{S^{\ast }}}\left( r\right)\right) ds\right) \left( F_{\mu _{S^{\ast }}}\left( r\right) -F_{\mu_{S}}\left( r\right) \right) dr.
\]
Suppose first $r^{\ast }\leq r$. If $W>0$ then $sF_{\mu _{S}}\left( r\right)+\left( 1-s\right) F_{\mu _{S^{\ast }}}\left( r\right) <F_{\mu _{S^{\ast}}}\left( r^{\ast }\right) \leq F_{\mu _{S^{\ast }}}\left( r\right)$ and therefore $F_{\mu _{S}}\left( r\right) <F_{\mu _{S^{\ast }}}\left( r^{\ast}\right) $, so the integrand is positive, or else $W=0$. For a lower bound we can therefore restrict the integration in $r$ to the interval $\left[0,r^{\ast }\right) $.
					
If $r<r^{\ast }$ then by (\ref{Condition A}) $sF_{\mu _{S}}\left( r\right)+\left( 1-s\right) F_{\mu _{S^{\ast }}}\left( r\right) <F_{\mu _{S^{\ast}}}\left( r\right) \leq F_{\mu _{S^{\ast }}}\left( r^{\ast }\right) $ so $W\left( sF_{\mu _{S}}\left( r\right) +\left( 1-s\right) F_{\mu _{S^{\ast}}}\left( r\right) \right) \geq W\left( F_{\mu _{S^{\ast }}}\left( r\right)\right) $, since $W$ is nonincreasing. The conclusion follows from (\ref{Condition A}). 
\end{proof}
								
We restate Assumption A and Theorem 1.
								
\textbf{Assumption A}. There exists $S_{0}\in \mathcal{S}$, $\delta >0$, $\beta \in \left( 0,1-\lambda \right) $ and a scale parameter $r^{\ast }\in\left( 0,1\right) $ (in units of squared euclidean distance), such that for every model $S\in \mathcal{S}$ satisfying $\Phi \left( \mu _{S}^{\ast}\right) >\Phi \left( \mu _{S_{0}}^{\ast }\right) +\delta $ we have $F_{\mu_{S}}\left( r\right) <\beta F_{\mu _{S_{0}}^{\ast }}\left( r\right) $ for all $r\leq r^{\ast }$. 
							
\begin{theorem}
Let $\mu ^{\ast },\nu \in \mathcal{P}\left( \mathbb{R}^{d}\right) $, $\mu=\left( 1-\lambda \right) \mu ^{\ast }+\lambda \nu $, and $\lambda \in\left( 0,1\right) $ and suppose there are $S_{0}$, $r^{\ast }$, $\delta >0$ and $0<\beta <1-\lambda $, satisfying Assumption A. Suppose that $W$ is nonzero on a set of positive Lebesgue measure, nonincreasing and $W\left(t\right) =0$ for $t\geq \zeta =F_{\mu _{S_{0}}}\left( r^{\ast }\right)$.
\end{theorem}
							
\begin{proof}
Let $S,S_{0}\in \mathcal{S}$ and assume that $\Phi \left( \mu _{S}^{\ast}\right) >\Phi \left( \mu _{S_{0}}^{\ast }\right) +\delta $. Then for $r\leq r^{\ast }$ Assumption A implies $F_{\mu _{S}}\left( r\right) \leq \beta F_{\mu _{S_{0}}^{\ast }}\left( r\right) \leq \frac{\beta }{1-\lambda }F_{\mu_{S_{0}}}\left( r\right) $, and the conditions on $W$ also imply that $W=0$ on $\left[ F_{\mu _{S^{\ast }}}\left( r^{\ast }\right) ,1\right] $. Thus Lemma \ref{Lemma Robust} with $a=\beta /\left( 1-\lambda \right) <1$ gives%

\[
\Phi _{W}\left( \mu _{S}\right) -\Phi _{W}\left( \mu _{S_{0}}\right) \geq \left( 1-\frac{\beta }{1-\lambda }\right) IF_{\max }\left( \mu_{S_{0}},W\right) >0.
\]
Thus, if $\Phi _{W}\left( \mu _{S}\right) -\Phi _{W}\left( \mu_{S_{0}}\right) <\left( 1-\frac{\beta }{1-\lambda }\right) IF_{\max }\left(
\mu _{S_{0}},W\right) $, we must have $\Phi \left( \mu _{S}^{\ast }\right)\leq \Phi \left( \mu _{S_{0}}^{\ast }\right) +\delta $. The condition (\ref	{Condition A}) is clearly always satisfied by the minimizer $S^{\dagger }\left( \mu \right) $ of $\Phi _{W}\left( \mu _{S}\right) $. 
\end{proof}
							
\subsection{Generalization}
\label{app:A2}

A second application of Lemma \ref{Lemma derivative} gives a Lipschitz property of $\Phi _{W}$ relative to the Wasserstein and Kolmogorov metrics for distributions with bounded support.
							
\begin{lemma}
\label{Lemma Lipschitz 1}
For $\rho _{1},\rho _{2}\in \mathcal{P}$ with support in $\left[ 0,R_{\max }\right] $ and $\left\Vert W\right\Vert_{\infty }<\infty $%

\[
\Phi _{W}\left( \rho _{2}\right) -\Phi _{W}\left( \rho _{1}\right)  \leq
\left\Vert W\right\Vert _{\infty }d_{\mathcal{W}}\left( \rho _{1},\rho_{2}\right)
\]
and
\[
\Phi _{W}\left( \rho _{2}\right) -\Phi _{W}\left( \rho _{1}\right)  \leq
R_{\max }\left\Vert W\right\Vert _{\infty }d_{K}\left( \rho _{1},\rho_{2}\right) .
\]
Here $d_{\mathcal{W}}\left( \rho _{1},\rho _{2}\right) =\left\Vert F_{\rho_{1}}-F_{\rho _{2}}\right\Vert _{1}$ is the 1-Wasserstein distance and $d_{K}\left( \rho _{1},\rho _{2}\right) =\left\Vert F_{\rho _{1}}-F_{\rho_{2}}\right\Vert _{\infty }$ the Kolmogorov-Smirnov distance.
\end{lemma}
\begin{proof}
From (\ref{Difference formula}) and Hoelder's inequality we get

\[
\Phi _{W}\left( \rho _{1}\right) -\Phi _{W}\left( \rho _{2}\right)=-\int_{0}^{\infty }\left( \int_{F_{\rho _{2}}\left( r\right) }^{F_{\rho_{1}}\left( r\right) }W\left( u\right) du\right) dr\leq 2\left\Vert W\right\Vert _{\infty }\int_{0}^{\infty }\left\vert F_{\rho _{1}}\left(r\right) -F_{\rho _{2}}\left( r\right) \right\vert dr.
\]
We can bound the integral either by $\left\Vert F_{\rho _{1}}-F_{\rho_{2}}\right\Vert _{1}=d_{\mathcal{W}}\left( \rho _{1},\rho _{2}\right) $, which gives the first inequality, or by

\[
\int_{0}^{R_{\max }}\left\vert F_{\rho _{1}}\left( r\right) -F_{\rho_{2}}\left( r\right) \right\vert dr\leq \left\Vert F_{\rho _{1}}-F_{\rho_{2}}\right\Vert _{\infty }\int_{0}^{R_{\max }}dr=R_{\max }d_{K}\left( \rho_{1},\rho _{2}\right) ,
\]
which gives the second inequality.
\end{proof}
							
The Lipschitz properties imply estimation and bias bounds for the plug-in estimator. 
							
\begin{corollary}
\label{Corollary DKW and Bias}
Let $\rho \in \mathcal{P}$ with support in $\left[ 0,R_{\max }\right] $ and $\left\Vert W\right\Vert _{\infty }<\infty $ and suppose that $\hat{\rho}$ is the empirical measure generated from $n$ iid observations $\mathbf{R}=\left( R_{1},...,R_{n}\right) \sim \rho ^{n}$ 
								
\[
\hat{\rho}\left( \mathbf{R}\right) =\frac{1}{n}\sum_{i=1}^{n}\delta _{R_{i}}.
\]
Then (i)

\[
\Pr \left\{ \left\vert \Phi _{W}\left( \rho \right) -\Phi _{W}\left( \hat{\rho}\left( \mathbf{R}\right) \right) \right\vert >t\right\} \leq 2\exp \left( \frac{-2nt^{2}}{R_{\max }^{4}\left\Vert W\right\Vert _{\infty }^{2}}\right).
\]
and (ii)

\[
\Phi _{W}\left( \rho \right) -\mathbb{E}\left[ \Phi _{W}\left( \hat{\rho}\left( \mathbf{R}\right) \right) \right] \leq \frac{R_{\max }\left\Vert	W\right\Vert _{\infty }}{\sqrt{2n}}.
\]
\end{corollary}
							
\begin{proof}
(i) By Lemma \ref{Lemma Lipschitz 1} and the Dvoretzky-Kiefer-Wolfowitz Theorem in the version of Massart \cite{Massart1990}

\[
\Pr \left\{ \left\vert \Phi _{W}\left( \rho \right) -\Phi _{W}\left( \hat{\rho}\left( \mathbf{R}\right) \right) \right\vert >t\right\} \leq \Pr\left\{ d_{K}\left( \rho ,\hat{\rho}\left( \mathbf{R}\right) \right) >\frac{t}{R_{\max }\left\Vert W\right\Vert _{\infty }}\right\} \leq 2\exp \left(\frac{-2nt^{2}}{R_{\max }^{2}\left\Vert W\right\Vert _{\infty }^{2}}\right) .
\]
								
(ii) Let $\mathbf{R}^{\prime }=\left( R_{1},...,R_{n}\right) $ be iid to $\mathbf{R}$. Then

\begin{eqnarray*}
\Phi _{W}\left( \rho \right) -\mathbb{E}\left[ \Phi _{W}\left( \hat{\rho}
\left( \mathbf{R}\right) \right) \right]  &\leq &\left\Vert W\right\Vert_{\infty }\mathbb{E}\left[ d_{\mathcal{W}}\left( \rho _{1},\hat{\rho}\left(\mathbf{R}\right) \right) \right]  \\
&=&\left\Vert W\right\Vert _{\infty }\mathbb{E}_{\mathbf{R}}\int_{0}^{R_{\max }}\left\vert \mathbb{E}_{\mathbf{R}^{\prime }}\left[ 
\frac{1}{n}\sum_{i}1_{\left[ R_{i}^{\prime },\infty \right) }\left( t\right)\right] -\left[ \frac{1}{n}\sum_{i}1_{\left[ R_{i},\infty \right) }\left(t\right) \right] \right\vert dt \\
&\leq &\frac{\left\Vert W\right\Vert _{\infty }}{n}\int_{0}^{R_{\max }}\mathbb{E}_{\mathbf{RR}^{\prime }}\left\vert \sum_{i}\left( 1_{\left[R_{i}^{\prime },\infty \right) }\left( t\right) -1_{\left[ R_{i},\infty\right) }\left( t\right) \right) \right\vert dt \\
&\leq &\frac{\left\Vert W\right\Vert _{\infty }}{n}\int_{0}^{R_{\max}}\left( \mathbb{E}_{\mathbf{RR}^{\prime }}\sum_{i}\left( 1_{\left[
R_{i}^{\prime },\infty \right) }\left( t\right) -1_{\left[ R_{i},\infty\right) }\left( t\right) \right) ^{2}\right) ^{1/2}dt \\
&=&\frac{\left\Vert W\right\Vert _{\infty }}{\sqrt{n}}\int_{0}^{R_{\max}}\left( \mathbb{E}_{R_{1}R_{1}^{\prime }}\left( 1_{\left[ R_{1}^{\prime},\infty \right) }\left( t\right) -1_{\left[ R_{1},\infty \right) }\left(
t\right) \right) ^{2}\right) ^{1/2}dt
\end{eqnarray*}%

by Jensens inequality and independence. But the expectation is just twice the variance of the Bernoulli variable $1_{\left[ R_{1},\infty \right)}\left( t\right) $, and therefore at most $1/2$. The result follows.
\end{proof}
						
Rephrasing part (i) of this corollary in terms of confidence windows we have, for any $\delta >0$ with probability at least $1-\delta $ that 

\[
\left\vert \Phi _{W}\left( \rho \right) -\Phi _{W}\left( \hat{\rho}\left(\mathbf{R}\right) \right) \right\vert \leq R_{\max }\left\Vert W\right\Vert_{\infty }\sqrt{\frac{\ln \left( 2/\delta \right) }{2n}}.
\]
For the weight function $W=\zeta ^{-1}1_{\left[ 0,\zeta \right] }$ the bound on the estimation error scales with $\zeta ^{-1}$, which is not surprising,	since we only consider a fraction $\zeta $ of the data. So for decreasing $\zeta $ the functional becomes more robust (because the influence $R_{\zeta }$ decreases) but it becomes more difficult to estimate. 
						
Restatement of Proposition 2.
						
\begin{proposition}
Assume the conditions of Theorem 1. Then%
\[
\Pr \left\{ \Phi \left( \mu _{\hat{S}\left( \mathbf{X}\right) }^{\ast}\right) \leq \Phi \left( \mu _{S^{\ast }}^{\ast }\right) +\delta \right\}\geq \Pr \left\{ 2\sup_{S\in \mathcal{S}}\left\vert \Phi _{W}\left( \mu_{S}\right) -\Phi _{W}\left( \hat{\mu}_{S}\left( \mathbf{X}\right) \right)\right\vert \leq \left( 1-\frac{\beta }{1-\lambda }\right) IC_{\max }\left(
\mu _{S^{\ast }},W\right) \right\} .
\]
 
\end{proposition}
						
\begin{proof}
\begin{eqnarray*}
\Phi _{W}\left( \mu _{\hat{S}\left( \mathbf{X}\right) }\right) -\Phi_{W}\left( \mu _{S^{\ast }}\right)  &\leq &\left( \Phi _{W}\left( \mu _{\hat{S}\left( \mathbf{X}\right) }\right) -\Phi _{W}\left( \hat{\mu}_{\hat{S}\left( \mathbf{X}\right) }\left( \mathbf{X}\right) \right) \right) +\left(\Phi _{W}\left( \hat{\mu}_{\hat{S}\left( \mathbf{X}\right) }\left( \mathbf{X}\right) \right) -\Phi _{W}\left( \hat{\mu}_{S^{\dagger }}\left( \mathbf{X}\right) \right) \right)  \\
&&+\left( \Phi _{W}\left( \hat{\mu}_{S^{\dagger }}\left( \mathbf{X}\right)\right) -\Phi _{W}\left( \mu _{S^{\dagger }}\right) \right) +\left( \Phi_{W}\left( \mu _{S^{\dagger }}\right) -\Phi _{W}\left( \mu _{S^{\ast}}\right) \right) .
\end{eqnarray*}
The second term and the last term are negative by the minimality properties of $\hat{S}\left( \mathbf{X}\right) $ and $S^{\dagger }$. The remaining terms are bounded by $2\sup_{S\in \mathcal{S}}\left\vert \Phi _{W}\left( \mu_{S}\right) -\Phi _{W}\left( \hat{\mu}_{S}\left( \mathbf{X}\right) \right)\right\vert $. Thus 

\begin{align*}
& \Pr \left\{ 2\sup_{S\in \mathcal{S}}\left\vert \Phi _{W}\left( \mu_{S}\right) -\Phi _{W}\left( \hat{\mu}_{S}\left( \mathbf{X}\right) \right)\right\vert \leq \left( 1-\frac{\beta }{1-\lambda }\right) IC_{\max }\left(\mu _{S^{\ast }},W\right) \right\}  \\
& \leq \Pr \left\{ \Phi _{W}\left( \mu _{\hat{S}\left( \mathbf{X}\right)}\right) -\Phi _{W}\left( \mu _{S^{\ast }}\right) \leq \left( 1-\frac{\beta}{1-\lambda }\right) IC_{\max }\left( \mu _{S^{\ast }},W\right) \right\}  \\
& \leq \Pr \left\{ \Phi \left( \mu _{\hat{S}\left( \mathbf{X}\right) }^{\ast}\right) \leq \Phi \left( \mu _{S^{\ast }}^{\ast }\right) +\delta \right\},
\end{align*}
where the last inequality follows from Theorem 1.
\end{proof}
						
\begin{lemma}
\label{Lemma hard threshold representation}
If $W=\zeta ^{-1}1_{\left[0,\zeta \right] }$ with $\zeta <1$, then for $\rho \in \mathcal{P}\left(\left[ 0,R_{\max }\right) \right) $ 

\[
\Phi _{W}\left( \rho \right) =\sup_{\lambda \in \left[ 0,R_{\max }\right]}\left\{ \lambda -\zeta ^{-1}\int_{0}^{\infty }\max \left\{ \lambda	-t,0\right\} d\rho \left( t\right) \right\} 
\]
\end{lemma}
						
\begin{proof}
Integration by parts gives 

\[
\int_{0}^{\infty }\max \left\{ \lambda -t,0\right\} d\rho \left( t\right)=\int_{0}^{\lambda }F_{\rho }\left( t\right) dt=\lambda F_{\rho }\left(\lambda \right) -\int_{0}^{\lambda }td\rho \left( t\right).
\]
The maximum of $\lambda -\zeta ^{-1}\int_{0}^{\lambda }F_{\rho }\left(t\right) dt$ is attained at $\zeta =F_{\rho }\left( \lambda \right) $, which shows $\lambda \leq R_{\max }$, and substitution gives

\begin{eqnarray*}
\sup_{\lambda \in \mathbb{R}}\left\{ \lambda -\zeta ^{-1}\int_{0}^{\infty}\max \left\{ \lambda -t,0\right\} d\rho \left( t\right) \right\} &=&\lambda -\zeta ^{-1}\left( \lambda F_{\rho }\left( \lambda \right)-\int_{0}^{\lambda }td\rho \left( t\right) \right)  \\
&=&\zeta ^{-1}\int_{0}^{F_{\rho }^{-1}\left( \zeta \right) }td\rho \left(t\right) =\int_{0}^{\infty }t\zeta ^{-1}1_{\left[ 0,\zeta \right] }\left(F_{\rho }\left( t\right) \right) d\rho \left( t\right)  \\
&=&\Phi _{W}\left( \rho \right).
\end{eqnarray*}%
 
\end{proof}
						
Restatement of Theorem 3.
						
\begin{theorem}
Let $W=\zeta ^{-1}1_{\left[ 0,\zeta \right] }$ and $\eta >0$. With probability at least $1-\eta $ in $\mathbf{X}\sim \mu ^{n}$ we have that 
\[
\sup_{S\in \mathcal{S}}\left\vert \Phi _{W}\left( \mu _{S}\right) -\Phi_{W}\left( \hat{\mu}_{S}\left( \mathbf{X}\right) \right) \right\vert \leq \frac{2}{\zeta n}\mathbb{E}_{\mathbf{X}}\mathcal{R}\left( \mathcal{F},\mathbf{X}\right) +\frac{R_{\max }}{\zeta \sqrt{n}}\left( 2+\sqrt{\frac{\ln\left( 2/\eta \right) }{2}}\right) ,
\]
where $\mathcal{R}\left( \mathcal{F},\mathbf{X}\right) $ is the Rademacher average 

\[
\mathcal{R}\left( \mathcal{F},\mathbf{X}\right) =\mathbb{E}_{\mathbf{\epsilon }}\left[ \sup_{S\in \mathcal{S}}\sum_{i=1}^{n}\epsilon _{i}d\left(X_{i},S\right) \right] 
\]
with independent Rademacher variables $\mathbf{\epsilon }=\left( \epsilon_{1},...,\epsilon _{n}\right) $.
\end{theorem}
						
\begin{proof}
Using Lemma \ref{Lemma hard threshold representation} we get with independent Rademacher variables $\mathbf{\epsilon }=\left( \epsilon_{1},...,\epsilon _{n}\right) $

\begin{eqnarray*}
&&\mathbb{E}\left[ \sup_{S\in \mathcal{S}}\Phi _{W}\left( \mu _{S}\right)-\Phi _{W}\left( \hat{\mu}_{S}\left( \mathbf{X}\right) \right) \right]  \\
&\leq &\zeta ^{-1}\mathbb{E}_{\mathbf{X}}\left[ \sup_{\lambda \in \left[0,R_{\max }\right] ,S\in \mathcal{S}}\int_{0}^{\infty }\max \left\{ \lambda	-t,0\right\} d\hat{\mu}_{S}\left( \mathbf{X}\right) \left( t\right)-\int_{0}^{\infty }\max \left\{ \lambda -t,0\right\} d\mu _{S}\left(t\right) \right]  \\
&=&\zeta ^{-1}\mathbb{E}_{\mathbf{X}}\left[ \sup_{\lambda \in \left[0,R_{\max }\right] ,S\in \mathcal{S}}\frac{1}{n}\sum_{i=1}^{n}\max \left\{\lambda -d\left( X_{i},S\right) ,0\right\} -\mathbb{E}_{X\sim \mu }\left[\max \left\{ \lambda -d\left( X,S\right) ,0\right\} \right] \right]  \\
&=&\frac{1}{\zeta n}\mathbb{E}_{\mathbf{XX}^{\prime }}\left[ \sup_{\lambda\in \left[ 0,R_{\max }\right] ,S\in \mathcal{S}}\sum_{i=1}^{n}\left( \max\left\{ \lambda -d\left( X_{i},S\right) ,0\right\} -\max \left\{ \lambda-d\left( X_{i}^{\prime },S\right) ,0\right\} \right) \right]  \\
&=&\frac{1}{\zeta n}\mathbb{E}_{\mathbf{XX}^{\prime }\mathbf{\epsilon }}\left[ \sup_{\lambda \in \left[ 0,R_{\max }\right] ,S\in \mathcal{S}}\sum_{i=1}^{n}\epsilon _{i}\left( \max \left\{ \lambda -d\left(X_{i},S\right) ,0\right\} -\max \left\{ \lambda -d\left( X_{i}^{\prime},S\right) ,0\right\} \right) \right]  \\
&\leq &\frac{2}{\zeta n}\mathbb{E}_{\mathbf{X\epsilon }}\left[ \sup_{\lambda\in \left[ 0,R_{\max }\right] ,S\in \mathcal{S}}\sum_{i=1}^{n}\epsilon_{i}\max \left\{ \lambda -d\left( X_{i},S\right) ,0\right\} \right]  \\
&\leq &\frac{2}{\zeta n}\mathbb{E}_{\mathbf{X\epsilon }}\left[ \sup_{\lambda\in \left[ 0,R_{\max }\right] ,S\in \mathcal{S}}\sum_{i=1}^{n}\epsilon_{i}\left( \lambda -d\left( X_{i},S\right) \right) \right]  \\
&\leq &\frac{2}{\zeta n}\mathbb{E}_{\mathbf{X\epsilon }}\left[ \sup_{S\in \mathcal{S}}\sum_{i=1}^{n}\epsilon _{i}d\left( X_{i},S\right) \right] +\frac{2}{\zeta n}\mathbb{E}_{\mathbf{\epsilon }}\left[ \sup_{\lambda \in \left[0,R_{\max }\right] }\lambda \sum_{i=1}^{n}\epsilon _{i}\right]  \\
&\leq &\frac{2}{\zeta n}\mathbb{E}_{\mathbf{X}}\mathcal{R}\left( \mathcal{F},\mathbf{X}\right) +\frac{2R_{\max }}{\zeta \sqrt{n}}.
\end{eqnarray*}
Here the third identity is a standard symmetrization argument, the second inequality the triangle inequality, followed by the contraction inequality for Rademacher averages, since $t\rightarrow \max \left\{ t,0\right\} $ is a contraction. Then we used the triangle inequality again. Now let $\Psi\left( \mathbf{X}\right) $ be the random variable $\sup_{S\in \mathcal{S}}\Phi _{W}\left( \mu _{S}\right) -\Phi _{W}\left( \hat{\mu}_{S}\left(\mathbf{X}\right) \right) $. It then follows from Lemma \ref{Lemma Lipschitz 1} and the bounded difference inequality that with probability at least $1-\eta $ we have $\Psi \left( X\right) \leq \mathbb{E}\Psi \left( X\right)+\zeta ^{-1}R_{\max }\sqrt{\ln \left( 1/\eta \right) /\left( 2n\right) }$.

Combined with above bound on $\mathbb{E}\Psi \left( X\right) $ this completes the proof.
\end{proof}

Theorem 4 follows directly from Theorems 2 and 5 in \cite{Maurer2019u} and from the bias bound, Corollary \ref{Corollary DKW and Bias} (ii). 
						
Restatement of Theorem 5.
						
\begin{theorem}
\label{Theorem Variance uniform bound copy(1)}
Under the conditions of the previous theorem, with probability at least $1-\eta $ in $\mathbf{X}\sim \mu^{n}$ we have that for all $S\in \mathcal{S}$
\[
\left\vert \Phi _{W}\left( \mu _{S}\right) -\Phi _{W}\left( \hat{\mu}_{S}\left( \mathbf{X}\right) \right) \right\vert \leq \sqrt{2V_{S}C}+\frac{6R_{\max }\left( \left\Vert W\right\Vert _{\infty }+\left\Vert W\right\Vert_{Lip}\right) C}{n}+\frac{\left\Vert W\right\Vert _{\infty }R_{\max }}{\sqrt{n}},
\]
where $V_{S}$ is the variance of the random variable $\Phi _{W}\left( \hat{\mu}_{S}\left( \mathbf{X}\right) \right) $, and $C$ is the complexity term

\[
C=kd\ln \left( 16n\left\Vert \mathcal{S}\right\Vert ^{2}/\eta \right) 
\]
if $\mathcal{S}$ is the set of sets with $k$ elements, or convex polytopes with $k$ vertices and $\left\Vert \mathcal{S}\right\Vert =\sup_{x\in S\in\mathcal{S}}\left\Vert x\right\Vert $, or%

\[
C=kd\ln \left( 16nR_{\max }^{2}/\eta \right) 
\]
if $\mathcal{S}$ is the set of set of $k$-dimensional subspaces.
\end{theorem}
						
\begin{proof}
For any fixed $S\in \mathcal{S}$ the L-statistic $\mathbf{x}\in \mathcal{X}^{n}\mapsto f_{S}\left( \mathbf{x}\right) :=\Phi _{W}\left( \hat{\mu}_{S}\left( \mathbf{x}\right) \right) $ is $\left( R_{\max }\left\Vert W\right\Vert _{\infty },R_{\max }\left\Vert W\right\Vert _{Lip}\right) $ -weakly interacting (see \cite{Maurer2018}) and therefore satisfies the following version of Bernstein's inequality (see \cite{Maurer2019b}, \cite{Maurer2018}): For $\eta \in \left( 0,1/e\right) $ with probability at least $1-\eta $ in $\mathbf{X\sim }\mu ^{n}$ we have
							
\[
\mathbb{E}\left[ f_{S}\right] -f_{S}\left( \mathbf{X}\right) \leq \sqrt{2V_{S}\ln \left( 1/\eta \right) }+R_{\max }\left( \frac{2\left\Vert W\right\Vert _{\infty }}{3}+\frac{3\left\Vert W\right\Vert _{Lip}}{2}\right) \frac{\ln \left( 1/\eta \right) }{n},
\]
where $\mathbb{E}\left[ f_{S}\right] $ and $V_{S}$ are expectation and variance of the random variable $f_{S}\left( \mathbf{X}\right) =\Phi_{W}\left( \hat{\mu}_{S}\left( \mathbf{X}\right) \right) $ respectively. We will make this bound uniform with a covering number argument.
							
Define a pseudo metric $d_{\mathcal{X}}$ on $\mathcal{S}$ by

\[
d_{\mathcal{X}}\left( S_{1},S_{2}\right) =\sup_{x\in \mathcal{X}}\left\vert d\left( x,S_{1}\right) -d\left( x,S_{2}\right) \right\vert .
\]
It follows from Lemma \ref{Lemma Lipschitz 1} that for every $\mathbf{x\in }\mathcal{X}^{n}$ we have 

\[
f_{S_{1}}\left( \mathbf{x}\right) -f_{S_{2}}\left( \mathbf{x}\right) \leq \left\Vert W\right\Vert _{\infty }d_{\mathcal{W}}\left( \hat{\mu}_{S_{1}}\left( \mathbf{x}\right) ,\hat{\mu}_{S_{2}}\left( \mathbf{x}\right)\right) \leq \left\Vert W\right\Vert _{\infty }d_{\mathcal{X}}\left(S_{1},S_{2}\right) .
\]
In particular $\left\vert \mathbb{E}\left[ f_{S_{1}}\right] -\mathbb{E}\left[f_{S_{2}}\right] \right\vert \leq \left\Vert W\right\Vert _{\infty }d_{\mathcal{X}}\left( S_{1},S_{2}\right) $ and 

\begin{eqnarray*}
\sqrt{V_{S_{1}}}-\sqrt{V_{S_{2}}} &=&\left\Vert f_{S_{1}}-\mathbb{E}\left[
f_{S_{1}}\right] \right\Vert _{L_{2}\left( \mu ^{n}\right) }-\left\Vert f_{S_{2}}-\mathbb{E}\left[ f_{S_{2}}\right] \right\Vert _{L_{2}\left( \mu^{n}\right) } \\
&\leq &\left\Vert f_{S_{1}}-f_{S_{2}}\right\Vert _{L_{2}\left( \mu^{n}\right) }+\left\vert \mathbb{E}\left[ f_{S_{1}}\right] -\mathbb{E}\left[f_{S_{2}}\right] \right\vert \leq 2\left\Vert W\right\Vert _{\infty }d_{%
\mathcal{X}}\left( S_{1},S_{2}\right).
\end{eqnarray*}
Now let $N=N\left( \mathcal{S},d_{\mathcal{X}},\epsilon \right) $ be the corresponding minimal covering number of $\mathcal{S}$ with $d_{\mathcal{X}}$-balls of radius $\epsilon $, and let $\mathcal{S}_{0}\subseteq \mathcal{S}$ be such that $\forall S\in \mathcal{S}$, $\exists S^{\prime }\in \mathcal{S}_{0}$ with $d_{R}\left( S,S^{\prime }\right) <1/n$ and $\left\vert \mathcal{S}_{0}\right\vert \leq N$. Then, abbreviating $R_{\max }\left( 2\left\Vert W\right\Vert _{\infty }/3+3\left\Vert W\right\Vert _{Lip}/2\right) $ with $C$, with probability at least $1-\eta $ in $\mathbf{X}$ that for every $S\in \mathcal{S}$

\begin{eqnarray*}
\mathbb{E}\left[ f_{S}\right] -f_{S}\left( \mathbf{X}\right)  &\leq &\mathbb{E}\left[ f_{S^{\prime }}\right] -f_{S^{\prime }}\left( \mathbf{X}\right) +	\frac{2\left\Vert W\right\Vert _{\infty }}{n}\leq \sqrt{2V_{S^{\prime }}\ln \left( N/\eta \right) }+\frac{C\ln \left( N/\eta \right) +2\left\Vert	W\right\Vert _{\infty }}{n} \\
&=&\sqrt{2V_{S}\ln \left( N/\eta \right) }+\frac{C\ln \left( N/\eta \right)+2\left\Vert W\right\Vert _{\infty }}{n}+\left( \sqrt{V_{S^{\prime }}}-\sqrt{V_{S}}\right) \sqrt{2\ln \left( N/\eta \right) } \\
&\leq &\sqrt{2V_{S}\ln \left( N/\eta \right) }+\frac{C\ln \left( N/\eta\right) +2\left\Vert W\right\Vert _{\infty }\sqrt{2\ln \left( N/\eta \right)}+2\left\Vert W\right\Vert _{\infty }}{n}.
\end{eqnarray*}
In the first inequality we used uniform approximation of $f_{S}$ by $f_{S^{\prime }}$, where $S^{\prime }$ is the nearest neighbour of $S$ in $\mathcal{S}_{0}$. The next line combines Bernstein's inequality with a union bound over $\mathcal{S}_{0}$. Finally we again approximate $\sqrt{V_{S^{\prime }}}$ by $\sqrt{V_{S}}$.
							
Next we bound the covering numbers $N\left( \mathcal{S},d_{\mathcal{X}},1/n\right) $, which we do separately for the case of uniformly bounded $\mathcal{S}$ and PSA. In case of the mean, k-means or sparse coding is easy to see that for $S_{1},S_{2}\in \mathcal{S}$ and any two respective enumerations $x_{i}$ and $y_{i}$ or enumerations of the extreme points 

\[
d_{\mathcal{X}}\left( S_{1},S_{2}\right) \leq 2\left\Vert \mathcal{S}\right\Vert H\left( S_{1},S_{2}\right) \leq 2\left\Vert \mathcal{S}\right\Vert \max_{i}\left\Vert x_{i}-y_{i}\right\Vert \text{. }
\]
It follows that $N\left( \mathcal{S},d_{\mathcal{X}},1/n\right) $ can be bounded by the covering number of a ball of radius $\left\Vert \mathcal{S}	\right\Vert ^{2}$ in a $kd$-dimensional Banach space. Use the standard result of Cucker and Smale \cite{Cucker2001} we have 

\[
N\left( \mathcal{S},d_{\mathcal{X}},1/n\right) \leq \left( 8n\left\Vert\mathcal{S}\right\Vert ^{2}\right) ^{kd}.
\]
For PSA we can use unit vectors spanning the subspaces and instead of $\left\Vert \mathcal{S}\right\Vert ^{2}$ we have the maximal squared norm in the support, so

\[
N\left( \mathcal{S},d_{\mathcal{X}},1/n\right) \leq \left( 8n\left\Vert \mathcal{X}\right\Vert ^{2}\right) ^{kd}.
\]
\[
kd\ln \left( 8n\left\Vert \mathcal{S}\right\Vert ^{2}/\eta \right) .
\]
Putting it all together and adding the bias bound $\Phi _{W}\left( \mu_{S}\right) -\mathbb{E}\left[ \Phi _{W}\left( \hat{\mu}_{S}\left( \mathbf{X}\right) \right) \right] \leq \left\Vert W\right\Vert _{\infty }R_{\max }/\sqrt{n}$ (Corollary \ref{Corollary DKW and Bias} (ii)) we get

\begin{align*}
& \Phi _{W}\left( \mu _{S}\right) -\Phi _{W}\left( \hat{\mu}_{S}\left(\mathbf{X}\right) \right)  \\
& \leq \sqrt{2\sigma ^{2}\left( \Phi _{W}\left( \hat{\mu}_{S}\left( \mathbf{X}\right) \right) \right) kd\ln \left( 8n\left\Vert \mathcal{S}\right\Vert^{2}/\eta \right) }+\frac{R_{\max }\left( 6\left\Vert W\right\Vert _{\infty}+\frac{3\left\Vert W\right\Vert _{Lip}}{2}\right) kd\ln \left( 8n\left\Vert \mathcal{S}\right\Vert ^{2}/\eta \right) }{n}+\frac{\left\Vert W\right\Vert_{\infty }R_{\max }}{\sqrt{n}}
\end{align*}%

The result follows from elementary estimates and algebraic simplifications. 
\end{proof}
					
\section{Algorithms}
\label{app:B}
Restatement of Lemma 7.

\begin{lemma}
For any $S \in \mathcal{S}$ and any $p \in {\rm Sym}_n$, if $\pi$ is the ascending ordering of the $d_S(x_i)$s, then $\phi_S({\bf x}, p) \ge \phi_S({\bf x}, \pi) = {\hat \Phi}_S({\bf x})$.
\end{lemma}

\begin{proof}
Writing $w\left( i\right) =W\left( \frac{\pi \left( i\right) }{n}\right) $
and $z_{i}=d_{S}\left( x_{\pi \left( i\right) }\right) $ it is enough to
show that the identity permutation is a minimizer of%
\[
r\left( p\right) =\sum_{i=1}^{n}w\left( p\left( i\right) \right) z_{i}\text{
for }p\in \text{Sym}_{n}
\]%
This follows from the following claim, which we prove by induction:

For $k\in \left\{ 1,...,n\right\} $ there is for every $p\in ~$Sym$_{n}$
some $p^{\prime }\in ~$Sym$_{n}$ such that $r\left( p^{\prime }\right) \leq
r\left( p\right) $ and $p^{\prime }\left( j\right) =j$ for all $1\leq j<k$.
The case $k=1$ holds trivially. If the claim holds for any $k\leq n-1$ then
there is $q\in ~$Sym$_{n}$ such that $r\left( q\right) \leq r\left( p\right) 
$ and $q\left( j\right) =j$ for all $1\leq j<k$. If $q\left( k\right) =\pi
\left( k\right) $ then the claim for $k+1$ clearly holds by defining $%
p^{\prime }:=q$. If $q\left( k\right) \neq k$ note first that both $q\left(
k\right) >k$ and $q^{-1}\left( k\right) >k$. Then define $p^{\prime }\left(
j\right) :=q\left( j\right) $ except for $p^{\prime }\left( k\right) :=k$
and $p^{\prime }\left( q^{-1}\left( k\right) \right) :=q\left( k\right) $.
Then $p^{\prime }\left( j\right) =j$ for all $1\leq j<k+1$ and 
\[
r\left( p^{\prime }\right) -r\left( p\right) \leq r\left( p^{\prime }\right)
-r\left( q\right) =\left( w\left( k\right) -w\left( q\left( k\right) \right)
\right) \left( z_{k}-z_{q^{-1}\left( k\right) }\right) \leq 0,
\]%
because the first term is non-negative (since $w$ is non-increasing) and the
second non-positive. So $r\left( p^{\prime }\right) \leq r\left( p\right) $
which proves the claim for the case $k+1$ and completes the induction.
\end{proof}

Restatement of Theorem 9.

\begin{theorem}
Minimizing ${\hat \Phi}_S({\bf x})$ for the case of \textsc{kmeans} when $k = 1$ and $W$ is the hard threshold is NP-Hard.
\end{theorem}

\begin{proof}

Notice that minimizing the ${\hat \Phi}_S({\bf x})$ in the case of $\textsc{kmeans}$ is equivalent to minimize the following function of a subset $C \subseteq X$ of size $\lfloor zn \rfloor$

$$
L(C) = \frac{1}{n} \sum_{x \in C} \| x_i - \mu_C \|_2^2
$$
where $\mu_C = mean(C)$ and return $\mu_C$. In what follow we will consider $L(C)$ as actually $L(C)n$ in order to remove the constant factor outside the objective and simplify the notation. The following lemma enables us to rewrite $L(C)$ in terms of pairwise distances. 

\begin{lemma}
Let $C \subseteq X$, then 

\begin{equation}
L(C) = \frac{1}{2|C|} \sum_{x, y \in C} \|x - y\|_2^2. 
\end{equation}
\label{L:pairwise}
\end{lemma}

\begin{proof}
Let $X$ and $Y$ two i.i.d. random variables supported on $C$, then

\begin{align*}
\Ex [\|X - Y\|_2^2] &= \Ex [\|X \|_2^2] + \Ex [\|Y \|_2^2] - 2\Ex [\langle X, Y\rangle] \\
&= \Ex [\|X \|_2^2] + \Ex [\|X \|_2^2] - 2\Ex [\|\Ex[X]\|_2^2] \\
&= 2 \Ex [\|X \|_2^2] - 2\Ex [\|\Ex[X]\|_2^2] = 2 \Ex [\|X - \Ex[X]\|_2^2]. 
\end{align*}
Now assume $X$ and $Y$ are independent samples from the uniform distribution on $C$, then

\begin{align*}
\Ex [\|X - \Ex[X]\|_2^2] &= \frac{1}{|C|}\sum_{x \in C} \|x- \mu_C\|_2^2 \\
&=  \Ex [\|X - Y\|_2^2]/2 = \frac{1}{2|C|^2}\sum_{x, y \in C} \|x - y\|_2^2    
\end{align*}
from which the thesis follows.
\end{proof}

We recall the definition of NP-hardness for optimization problems. 

\begin{definition}
A computational problem $\Pi$ is said NP-hard (optimization) if and only if the related decision problem $\Pi_D$ is NP-hard. Assume $\Pi$ is defined as the problem of minimizing a function $f_X(\mu)$ defined by an input instance $X$ if the minimum exists, then $\Pi_D$ is defined as the problem of determining, given in input $X$ and a rational number $c$, whether there exist an assignment to the variables $\mu$ such that $f_X(\mu) \le q$.
\end{definition}

In order to show hardness of an optimization problem $\Pi$, it is enough to show hardness of the related decision problem $\Pi_D$. For this reason, the following will be useful.

\begin{definition}
\textsc{decision robust 1-means}

\begin{itemize}

\item[] \textit{Input:} Points $X =\{x_1, \ldots ,x_n\} \subset \R^d$, an integer $h$ and a rational number $c$.

\item[] \textit{Output:} \emph{Yes} if there exist a $C \subseteq X$ such that $|C| = h$ and $L(C) \le c$, \emph{No} otherwise.
\end{itemize}

\end{definition}

To prove the theorem we will reduce \textsc{$\nicefrac{n}{2}$-clique} to the decision version \textsc{robust 1-means} via a polynomial time algorithm. Since \textsc{$\nicefrac{n}{2}$-clique} is NP-complete, hardness for \textsc{robust 1-means} will follow.

\begin{definition}
\textsc{$\nicefrac{n}{2}$-clique}

\begin{itemize}

\item[] \textit{Input:} A simple undirected connected graph $G=(V, E)$ with $|V| = n$. 

\item[] \textit{Output:} \emph{Yes} if $G$ contains a clique of size $\nicefrac{n}{2}$, \emph{No} otherwise.

\end{itemize}

\end{definition}

Given an instance of $n/2$-\textsc{clique} in the form of a graph $G = (V, E)$ with $n$ vertices, we create an instance of \textsc{robust 1-means} $\Pi_D(G)$ which is equivalent to $G$. Let $A$ denote the symmetric $n \times n$ adjacency matrix of $G$, i.e. $A_{ij} = 1$ iff $(i, j) \in E$ otherwise $A_{ij} = 0$. Consider the graph embedding given by the map $\phi: V \rightarrow \R^n$ such that $\phi(i) = A_{i:} +  n e_i$, where $A_{i:}$ denotes the $i$-th row of $A$ and $e_i$ denotes the $i$-th vector of the canonical basis of $\R^n$. Given $G$ we build an instance of \textsc{robust 1-means} by setting $X = \{\phi(1),\ldots,\phi(n)\}$, $h=n/2$ and $c = m(2n^2-3n)$, where we set $m = \binom{n}{2}$ as a shortcut. Notice that it takes $O(n)$ to build such instance. The following lemma finishes the proof by showing the aforementioned equivalence. 

\begin{lemma}
$G$ is a Yes instance iff $\Pi_{D}(G)$ is a Yes instance.
\end{lemma}

\begin{proof}
Assume that $G$ is a \emph{Yes} instance, i.e. $G$ contains at least clique of size $\nicefrac{n}{2}$. Notice that for any $(i, j) \in E$ it holds that

$$
\|\phi(i)-\phi(j)\|_2^2 \le  (n-1)^2+ (n-1)^2 = 2n^2 - 4n + 2 \le 2n^2 - 3n
$$
while for any $(i, j) \notin E$ it holds
$$
\|\phi(i)-\phi(j)\|_2^2 \ge 2n^2.
$$
If $\{c_1,\ldots,c_{\nicefrac{n}{2}}\}$ are the vertices in the clique, the cost $L(C)$, by \Cref{L:pairwise}, of the subset $C = \{\phi(c_1),\ldots,\phi(c_{\nicefrac{n}{2}})\}$ is at most $c$, since in such clique contains exactly $m$ edges.

Now suppose that $\Pi_D(G)$ admits a cost of at most $c$. Lets denote by $C$ the subsets of $X$ achieving such cost, then the associated vertices must form a clique otherwise at least one of the $m$ distance will be larger than $2n^2$ leading to a cost larger of $c$.
\end{proof}

Thus if we could solve in polynomial time \textsc{decision robust 1-means} we could solve in polynomial time \textsc{$\nicefrac{n}{2}$-clique}.

\end{proof}

\begin{table}[t!]
\centering
\begin{tabular}{||c | c | c c c||} 
\hline
Dataset & $k$ & RKM & SD & \textsc{k-means++} \\ [0.5ex] 
\hline\hline
FMNIST & 2 & \textbf{25.98} & 33.17 & 34.39 \\ 
EMNIST & 2 & 38.41 & \textbf{37.78} & 40.29 \\
cifar10 & 4 & $\bf{31.07} \times \bf{10^4}$ & $96.42 \times 10^5$ & $95.57 \times 10^5$ \\
Victorian & 5 & \textbf{1.64} & 1.66 & 1.76 \\
Iris & 1 & \textbf{0.32} & 3.71 & 4.75 \\  
\hline
\end{tabular}
\caption{Experimental results for the case of \textsc{k-means} clustering. In all the experiments $\zeta = 0.5$. In each row, the performance in bold corresponds to the winning algorithm.}
\label{table:1}
\end{table}

\section{Experiments}
\label{app:C}
In this section we discuss the additional experimental results we obtained with our method in the case of \textsc{k-means} clustering. We tested RKM, SD and standard \textsc{k-means++} with the $\zeta = 0.5$, $r = 30$, and $T = 100$. Due to its cubic runtime, SD is slow even on moderate-sized datasets. Thus, we considered the randomized version of SD with $M$ equals to the size of the training set. For this method, we repeated each experiment 5 times and reported the average reconstruction error on the test data (standard deviations resulted to be negligible in all cases).

In the following we describe each dataset, but Fashion MNIST whose experiment has already been described in the main body.

\paragraph{EMNIST.} This dataset consists of about 814000 $28 \times 28$ images of digits, lowercase and uppercase letters from the English alphabet arranged in 62 classes. The training data were generated by sampling 1000 0s and 1000 1s as inliers and sampling 33 points from each other class as outliers. We used $k=2$ clusters. The test data consist of all the 0s and 1s in the test set and has a size of about 2000.

\paragraph{cifar10.} The dataset consists of about 60000 $100 \times 100$ images from 10 classes: airplanes, cars, trucks, ships, dogs, cats, frogs, horses, birds and deer. The training data were generated by sampling 1000 points from each of the vehicle classes as inliers and 300 points from each of the animal classes as outliers. We used $k=4$ clusters. The test data consist of all the vehicle images from the test set and has size of about $4000$.

\paragraph{Victorian.} This dataset consists of 4500 texts from 45 authors of English language from Victorian Era, 100 texts from each author. The data have been processed as in \cite{AhmadianE0M19} and is made of 10 features. The training data were generated by sampling 50 points from each of one of the first 5 authors in the dataset as inliears and 5 points from each other class as outliers. We used $k=5$. The test data consist of the remaining 50 points from each of the inlier authors and has a size of about $250$.

\paragraph{Iris.} This dataset consists of 150 records of iris flowers. Each record contains 4 features: sepal length, sepal width, petal length and petal width. There classes. The training data were generated by sampling 30 points from the \emph{iris-setosa} class as inliear and 15 points from each other class as outliers. We used $k=1$. Since the training set is small sized, we used exact version for SD. The test data consist of all the remaining \emph{iris-setosa} points and has a size of about $20$.

\end{document}